\documentclass{article}

\usepackage{microtype}
\usepackage{graphicx}
\usepackage{soul}
\usepackage{subfigure}
\usepackage{booktabs} 

\usepackage{hyperref}

\usepackage[accepted]{icml2022}

\usepackage{amsmath}
\usepackage{amssymb}
\usepackage{mathtools}
\usepackage{amsthm}
\usepackage{multirow}
\usepackage{tabularx}
\usepackage{caption}
\usepackage{rotating}
\usepackage{color}
\definecolor{Gray}{gray}{0.9}
\usepackage{colortbl}
\usepackage{mathrsfs}
\usepackage{bbding}

\usepackage[capitalize,noabbrev]{cleveref}

\theoremstyle{plain}
\newtheorem{theorem}{Theorem}[section]

\newtheorem{lemma}[theorem]{Lemma}

\theoremstyle{definition}

\newtheorem{assumption}[theorem]{Assumption}
\theoremstyle{remark}

\usepackage[textsize=tiny]{todonotes}
\newcommand{\hzz}[1]{{\color{magenta}{#1}}}
\usepackage{algpseudocode}
\icmltitlerunning{AttNS: Attention-Inspired Numerical Solving For Limited Data Scenarios}

\begin{document}

\twocolumn[
\icmltitle{AttNS: Attention-Inspired Numerical Solving For Limited Data Scenarios}






\begin{icmlauthorlist}
\icmlauthor{Zhongzhan Huang$^\dag$}{sysu,pc}
\icmlauthor{Mingfu Liang}{nw}
\icmlauthor{Shanshan Zhong}{sysu}
\icmlauthor{Liang Lin}{sysu,pc}
\end{icmlauthorlist}

\icmlaffiliation{sysu}{Sun Yat-sen University, China}
\icmlaffiliation{pc}{Peng Cheng Laboratory, China}
\icmlaffiliation{nw}{Northwestern University, USA. $\dag$ Informal short-term visiting at Peng Cheng Laboratory for two weeks}
\icmlcorrespondingauthor{Liang Lin}{linliang@ieee.org}
\icmlkeywords{Machine Learning, ICML}

\vskip 0.3in
]



\printAffiliationsAndNotice{} 

\begin{abstract}

We propose the attention-inspired numerical solver (AttNS), a concise method that helps the generalization and robustness issues faced by the AI-Hybrid numerical solver in solving differential equations due to limited data. AttNS is inspired by the effectiveness of attention modules in Residual Neural Networks (ResNet) in enhancing model generalization and robustness for conventional deep learning tasks. Drawing from the dynamical system perspective of ResNet, We seamlessly incorporate attention mechanisms into the design of numerical methods tailored for the characteristics of solving differential equations. Our results on benchmarks, ranging from high-dimensional problems to chaotic systems, showcase AttNS consistently enhancing various numerical solvers without any intricate model crafting. Finally, we analyze AttNS experimentally and theoretically, demonstrating its ability to achieve strong generalization and robustness while ensuring the convergence of the solver. This includes requiring less data compared to other advanced methods to achieve comparable generalization errors and better prevention of numerical explosion issues when solving differential equations.  \hzz{https://github.com/dedekinds/NeurVec}.

\end{abstract}

\section{Introduction}
\label{sec:intro}

The AI-Hybrid numerical solvers (AHS) are emerging methods~\cite{mishra2018machine,san2018extreme,bar2019learning,kochkov2021machine,subel2021data,bruno2022fc,list2022learned,dresdner2022learning,frezat2022posteriori,huang2023fast,xu2023} for solving differential equations, designed to integrate classical numerical techniques with deep learning technology. These approaches leverage high-resolution data to learn corrections~\cite{um2020solver} for low-resolution numerical solutions. They aim to combine the strengths of both techniques – the reliability of classical numerical methods and the expressive power of neural networks – to mitigate the inherent speed-accuracy trade-off in solving differential equations.
For example, for an ordinary differential equation (ODE) 
\vspace{-0.2cm}
\begin{equation}
\text{d}\mathbf{u(t)}/\text{d}t = \mathbf{f}[\mathbf{u}(t)], \mathbf{u}(0) = \mathbf{c}_0,
    \label{eq:ode}
\vspace{-0.2cm}
\end{equation}
where $\mathbf{u(t)}$ is a time-dependent $d$-dimensional state and the $\mathbf{c}_0 \in \mathbb{R}^d$ is initial condition. The forward AHS of Eq.~(\ref{eq:ode}) can usually ~\cite{dresdner2022learning} be written as
\begin{equation}
    \mathbf{u}_{n+1} = \mathbf{u}_{n} + \underbrace{S(\mathbf{f}, \mathbf{u}_{n}, \Delta t_c)\Delta t_c}_{\text{Low-resolution term}} + \underbrace{\text{Net}(\mathbf{u}_{n}|\mathbf{\phi},\mathbb{D}_f)}_{\text{Correction term}},
    \label{eq:aihybrid}
    \vspace{-0.2cm}
\end{equation}
where $\text{Net}(\cdot)$ is a neural network with learnable parameters $\phi$, $S(\cdot)$ is a numerical integration scheme $\mathbb{D}_f$ denotes the high-resolution data, i.e., the data with high-precision generated by fine step size $\Delta t_f$. 
A coarse step size, $\Delta t_c$, accelerates computation, and the associated decrease in accuracy is compensated by the correction term $\text{Net}(\mathbf{u}_{n}|\mathbf{\phi},\mathbb{D}_f)$ trained by $\mathbb{D}_f$, achieving a good balance between speed and accuracy in solving differential equations.

However, the performance of current AHS is heavily reliant on the quantity of high-quality training data due to the data-driven deep learning approach adopted by the correction term, and \textbf{the obtaining of sufficient training data is very expensive}. For instance, if we use data with $\Delta t_f=0.001$ to train the correction term for $\Delta t_c=0.1$, generating a single trajectory of data would require over 100$\times$ additional computations, which are hard to be accelerated by parallelization due to the iterative nature of Eq.(\ref{eq:aihybrid}). Besides, the high demand for data imposed by complex equations \cite{huang2023fast} further exacerbates the acquisition cost.   Given these challenges, we ask a critical question:

\ul{\textit{Can we improve the AHS for effective computations even in \textbf{limited data scenarios}? }}

\begin{figure*}[t]
\centering
\includegraphics[width=0.99\linewidth]{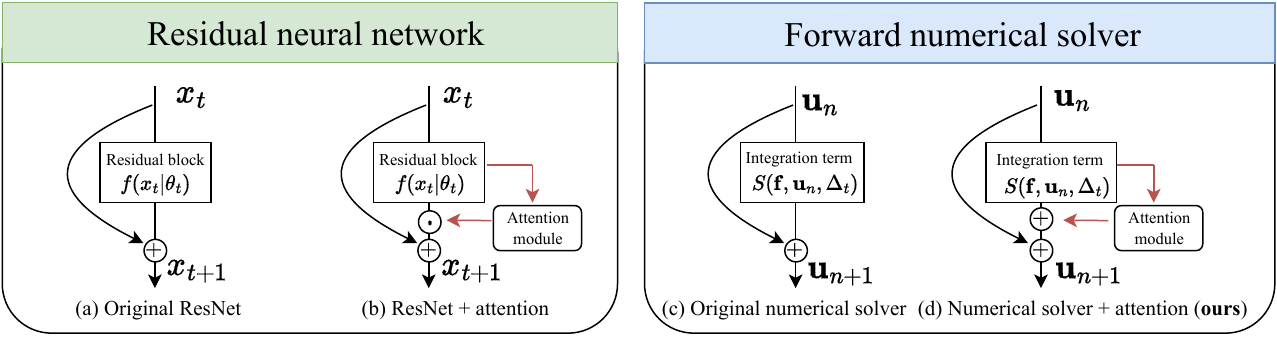}
\caption{The correspondence of (a) ResNet and (c) forward numerical solver. (d) is a numerical solver with the attention mechanism (ours) inspired by the structure of (b) ResNet with attention. $\odot$ is element-wise multiplication and $\oplus$ is the addition operator. The use of $\oplus$ is \textbf{tailored for} solving differential equations and see Section \ref{sec:methods} for details. }
\label{fig:jg2}
\end{figure*}
To answer this question, in this paper, we conduct a comprehensive analysis using the ODE shown in Eq.~(\ref{eq:ode}) as examples. First, we present the specific challenges for AHS in limited data scenarios in Section \ref{sec:challenges}. In addition to the inherent generalization issues of neural networks due to limited data, robustness issues~\cite{liang2022stiffnessaware} related to chaos in solving differential equations need to be carefully considered. In response to these challenges, we note that the attention mechanism module in Residual Neural Network (ResNet) has been widely validated to effectively enhance the model's generalization and robustness~\cite{hu2018squeeze,huang2020dianet,woo2018cbam,2020ECA,liang2020instance,Zhong2023ASRAS} in conventional deep learning tasks (see Section \ref{sec:attention}). Also, considering the correspondence between ResNet and numerical solvers established in the dynamical system perspective of ResNet (see Section \ref{sec:dynamical} and Fig.\ref{fig:jg2}), we follow the characteristics of differential equation solving and seamlessly integrate the attention mechanism module into the numerical solver, proposing a simple yet effective Attention-Inspired Numerical Solver (AttNS) in Section \ref{sec:methods}.

Through experiments with several standard benchmarks of high-dimensional or chaotic systems, including the spring-mass system, the elastic pendulum, and the $K$-link pendulum, the results in Section \ref{sec:exps} show that AttNS consistently outperforms existing state-of-the-art AHS in limited data scenarios, without the need for intricate model crafting. Next, in Section \ref{sec:dis}, we conduct experimental and theoretical analyses of AttNS, proving its strong generalization and robustness abilities akin to the attention mechanism in conventional deep learning tasks, while ensuring solver convergence. This includes achieving equivalent generalization errors with less data compared to other advanced methods and better preventing numerical explosion issues in solving differential equations. Finally, we discuss the limitations of AttNS. We summarize our contributions as follows:
\vspace{-0.2cm}
\begin{itemize}
\item For effective AHS in limited data scenarios, we incorporate the attention mechanism into AHS and proposed a simple-yet-effective method AttNS. The standard benchmarks show that AttNS consistently outperforms other advanced AHS methods.
\vspace{-0.2cm}
\item We conducted experimental and theoretical analyses for AttNS, demonstrating its strong generalization and robustness capabilities while ensuring solver convergence. Finally, we discuss the limitations of AttNS.
\end{itemize}

\section{Preliminaries and Related Works}
\label{sec:preliminary}

\subsection{The Solving of Differential Equation}
For a given equation as shown in Eq.~(\ref{eq:ode}), classical forward numerical methods, such as the Euler method \cite{shampine2018numerical}, Runge-Kutta method \cite{butcher2016numerical}, etc., solve it using the iterative formula 
\vspace{-0.2cm}
\begin{equation}
\mathbf{u}_{n+1} = \mathbf{u}_{n} + S(\mathbf{f}, \mathbf{u}_{n}, \Delta t)\Delta t, \quad \mathbf{u}_0 = \mathbf{c}_0.
    \label{eq:forward}
    \vspace{-0.2cm}
\end{equation}
Different methods have different $S$. For example, for the Euler method \cite{shampine2018numerical}, we have $S(\mathbf{f}, \mathbf{u}_{n}, \Delta t)\Delta t=\mathbf{f}(\mathbf{u}_{n})\Delta t$, where $\Delta t$ is a given step size, and $\mathbf{u}_{n}\in \mathbb{R}^d$ is an approximated solution at time $\sum_{i=0}^n \Delta t$. The AI-Hybrid numerical solver (AHS) considered in this paper, as shown in Eq.~(\ref{eq:aihybrid}), introduces a correction term to help high-precision solving of the differential equation even with larger step sizes which makes solving faster. For other methods utilizing deep learning techniques for solving, such as direct fitting using neural networks ~\cite{geneva2022transformers,liang2021solving}, operator-based methods focusing on replacing differential operators ~\cite{li2021fourier,lu2021learning}, neuralODE-based methods~\cite{chen2018neural,kidger2021hey,dupont2019augmented}, Hamiltonian-based methods ~\cite{greydanus2019hamiltonian,chen2019symplectic,liang2022stiffnessaware}, and physics-informed neural networks ~\cite{choudhary2020physics,raissi2019physics,ji2021stiff,cai2021physics}, they're \ul{\textbf{hard to conduct a fair comparison with AHS due to their fundamentally different paradigms}}. Therefore, this paper primarily focuses on AHS instead of other deep learning-based solving methods.

\vspace{-0.2cm}
\subsection{Dynamical System Perspective of ResNet}
\label{sec:dynamical}

As shown in Fig.\ref{fig:jg2}(a), the simplified residual block in ResNet~\cite{he2016deep} can be written as
\vspace{-0.2cm}
\begin{equation}
    x_{t+1} = x_t + f(x_t|\theta_t),
    \label{eq:resnet}
    \vspace{-0.2cm}
\end{equation}
where $x_t \in \mathbb{R}^d$ is the input of neural network $f(\cdot|\theta_t)$ with the learnable parameters $\theta_t$ in $t$-${\rm th}$ block . Several recent studies \cite{weinan2017proposal,queiruga2020continuous,zhu2022convolutional,meunier2022dynamical} have uncovered valuable connections between residual blocks and dynamic systems. i.e., the residual blocks can be interpreted as one step of forward numerical methods in Eq.~(\ref{eq:forward}) and Fig.\ref{fig:jg2}(c). The initial condition $\mathbf{u}_0 = \mathbf{c}_0$ in Eq.~(\ref{eq:forward}) corresponds to the initial input $x_0$ of the network, and $\mathbf{u}_t$ corresponds to the input feature $x_t$ in $t$-${\rm th}$ block. 
The output of neural network $f(\cdot|\theta_t)$ in $t$-${\rm th}$ block can be regarded as an integration $S(\mathbf{u}_t,\mathbf{f},\Delta t)$ with step size $\Delta t$ and numerical integration scheme $S$. Given the above connection, some dynamical system theories \cite{chang2017multi,chen2018neural,huang2022layer,DBLP:conf/icml/LuZLD18} can be transferred to the analysis of ResNet.

\vspace{-0.3cm}

\subsection{The Challenges for AHS in Limited Data Scenarios}
\label{sec:challenges}

We first briefly define limited data scenarios (LDS). Generally, the performance of AHS is positively correlated with the training data size~\cite{huang2023fast}, and reaches saturation at data quantity $N_b$, i.e., further increasing data size hardly boosts performance. In this paper, LDS are situation that the data size is $\lfloor pN_b \rfloor$, where $1\leq p \leq 0.5$. In fact, the LDS easily occurs as $N_b$ required for a given equation is not known before solving and generating data. A sound way is to make sure, as much as possible, that the solver maintains a good enough solving performance even in LDS, i.e., the goal of this paper. The main challenges of this goal are:

 (1) \ul{\textbf{Generalization issue}}. When the training data size is insufficient, the deep learning module in AHS suffers from the inherent generalization problem ~\cite{zhang2021understanding,neyshabur2017exploring}, i.e., poor prediction of the inputs outside the training data distribution. This results in a trained correlation term that cannot accurately compensate for the accuracy loss caused by the low-resolution term. 

 (2) \ul{\textbf{Robustness issue}}. Many differential equations are chaotic~\cite{greydanus2019hamiltonian}, i.e. sensitive to the small perturbations in inputs, which makes it hard for AHS to fundamentally model the ubiquitous and elusive randomness of chaos with only limited data, leading to poor generalization~\cite{abu2012learning}. It is easy to accumulate and propagate error for a non-robust AHS during the iterative solution process in Eq.~(\ref{eq:aihybrid}), leading to a numerical explosion.

\vspace{-0.3cm}
\subsection{Attention Mechanism for ResNet}
\label{sec:attention}

For the residual block mentioned in Eq.~(\ref{eq:resnet}), the attention mechanism\footnote{Unlike Transformer based attention~\cite{vaswani2017attention}, the attention mechanisms considered in this paper are for ResNet, such as SENet~\cite{hu2018squeeze}. } can be formulated as 
\vspace{-0.2cm}
\begin{equation}
x_{t+1} = x_{t} + f(x_t|\theta_t) \odot \alpha_t ,
    \label{eq:att}
    \vspace{-0.2cm}
\end{equation}
where $\alpha_t = Q[f(x_t|\theta_t)]$ is $t$-${\rm th}$ attention module~\cite{hu2018squeeze,woo2018cbam} and $\odot$ is element-wise multiplication, which also shown in Fig.\ref{fig:jg2}(b).
many recent works ~\cite{hu2018squeeze,huang2020dianet,woo2018cbam,2020ECA,liang2020instance, Zhong2023ASRAS} validate the effectiveness of these attention modules to improve model's generalization and robustness in \textbf{conventional deep learning tasks}, such as image classification, detection, segmentation, style migration, etc. Specifically,

 (1) \ul{\textbf{Improving generalization}}. The attention mechanism has the capability to adaptively learn the weights of features across various tasks and input data~\cite{qin2021fcanet,huang2020dianet,2020ECA}. This flexibility enables the model to focus on task-relevant information, adapt to new data, and effectively capture intricate relationships, ultimately enhancing the generalization performance of various conventional deep learning tasks.

 (2) \ul{\textbf{Enhancing robustness}}. The attention mechanism has been theoretically proven to mitigate small perturbations in the input, thereby improving robustness~\cite{Zhong2023ASRAS}. Specifically, let $\epsilon$ represent the perturbation from noise, satisfying $\|x_0^\epsilon - x_0\| = \epsilon$. The upper bound of the error $\|x_t^\epsilon - x_t\|$ at the $L$-th layer, caused by $x_0^\epsilon$ as input, is given by $\epsilon\prod_{t=1}^{L-1} \left(1+\alpha_t\|\theta_t\|_2\right)$ according to Eq~(\ref{eq:att}), where the adaptive control is exerted by the attention weights $\alpha_t$. This perspective is supported by experimental evidence in tasks such as style transfer and noise attacks~\cite{liang2020instance,Zhong2023ASRAS}.

\vspace{-0.3cm}
\section{Method}
\label{sec:methods}
As mentioned in Section \ref{sec:attention}, attention mechanisms can help model generalization and robustness issues in conventional deep learning tasks. These issues, as revealed in Section \ref{sec:challenges}, happen to be the primary challenges for AHS in limited data scenarios. Therefore, it is natural for us to consider seamlessly transferring the attention mechanism to the numerical solver by exploiting the correspondence between the forward numerical solver and ResNet mentioned in Section \ref{sec:dynamical}. Some previous works also considered using attention to help solution~\cite{geneva2022transformers,takamoto2023learning,rodriguez2022physics,hemmasian2023reduced}, \textbf{but they are not for AHS and not the attention mechanisms in Eq.~(\ref{eq:att})}.

\begin{figure*}[t]
\centering
\includegraphics[width=0.99\linewidth]{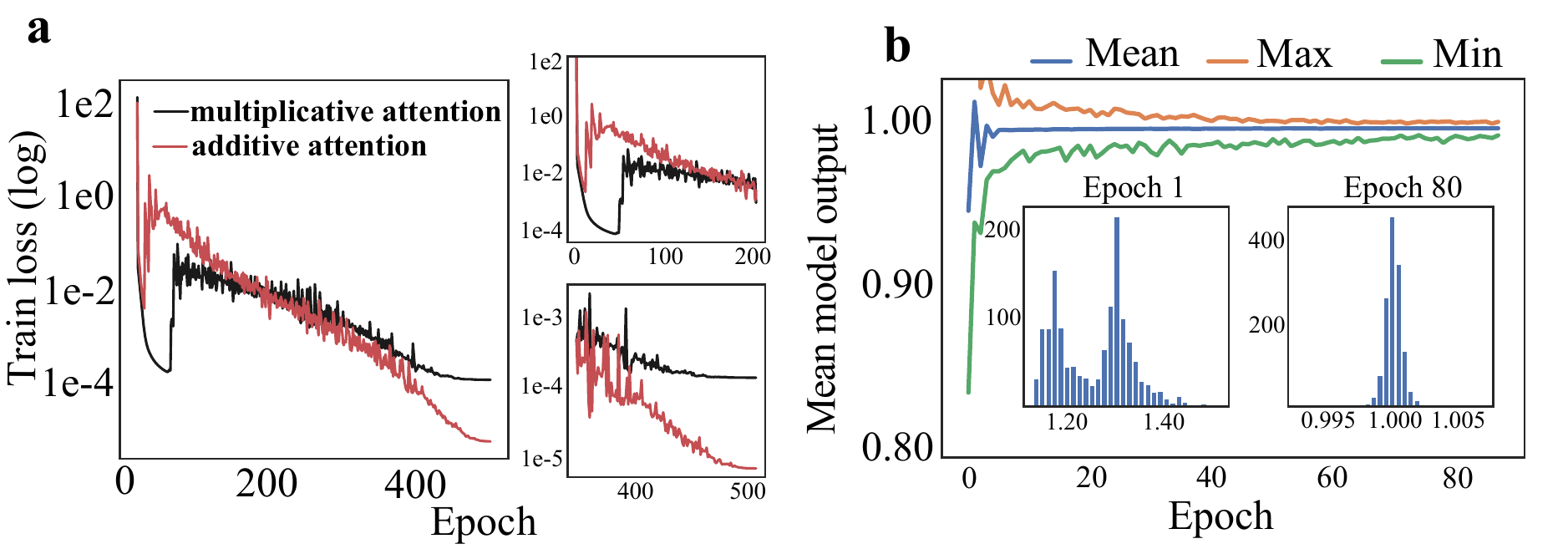}
\caption{\textbf{a}. The loss curves for two kinds of AttNS-m and AttNS. The loss minimization of multiplicative attention is fast at first and then slow during the last epochs, which have a local minimum with a large loss; \textbf{b}. The mean of attention value $Q[\hat{S}|\phi]$ (blue) while using Eq.(\ref{eq:mult}), which quickly converges to the vicinity of constant 1 during training.}
\label{fig:multi}
\end{figure*}

\textbf{(1) A straightforward baseline (AttNS-m).} According to  Eq.~(\ref{eq:att}), the straightforward numerical solving with attention can be set as Eq.~(\ref{eq:mult}) with multiplication operator.
\vspace{-0.15cm}
\begin{equation}
\text{AttNS-m:} \quad \hat{\mathbf{u}}_{n+1} = \hat{\mathbf{u}}_{n} + \hat{S}\Delta t_{c} {\color{blue} \odot Q[\hat{S}|\phi]},
    \label{eq:mult}
    \vspace{-0.15cm}
\end{equation}
where $\hat{S} = S(\mathbf{f}, \hat{\mathbf{u}}_{n}, \Delta t_c)$ and $\phi$ is the learnable parameters of attention module $Q[\cdot|\phi]$. The training loss is defined as the $\ell_2$-Squared distance between the estimated trajectory $\hat{\mathbf{u}} = [\hat{\mathbf{u}}_1,...,\hat{\mathbf{u}}_{N}]$ and the ground truth trajectory from the dataset $\mathbb{D}(\text{traj}(\mathbf{u}))$ following previous works~\cite{dresdner2022learning}:
\begin{equation}
R_e = \lambda\cdot\|\hat{\mathbf{u}} - \mathbf{u}\|_2^2/N,
    \label{eq:loss}
\end{equation}
where $\lambda$ is a penalty coefficient. $\lambda$ can alleviate the problem that the value of $R_e$ is too small due to the magnitude of $\mathbf{u}$ in some specific differential equations being too small. This is because, in such a scenario, the gradient will be small enough and affect the optimization of the learnable parameter $\phi$. For $\hat{S} \in \mathbb{R}^d$, the architecture of the attention module~\cite{xu2023} in Eq.~(\ref{eq:mult}) is
\begin{equation}
Q[\hat{S}|\phi] = \mathbf{W_h}\circ \mathbf{a}\circ\cdots \circ \mathbf{W_2}\circ \mathbf{a}\circ \mathbf{W_1}[\hat{S}],
    \label{eq:q}
\end{equation}
where $\mathbf{a}$ is rational activation function \cite{boulle2020rational}; $\mathbf{W_i},i $ = $ 2,\cdots,h-1$ are $d_1\times d_1$ matrices, $\mathbf{W_h} \in \mathbb{R}^{d\times d_1}$ and $\mathbf{W_1} \in \mathbb{R}^{d_1 \times d}$. We set $d_1=1024$ and $h=2$ by default. In Section \ref{sec:dis}, we will further discuss the concise design in Eq.~(\ref{eq:q}) is tailored for efficient solving of AHS, on the one hand, the small number of parameters enables faster inference speed, and on the other hand, it can be proven that this architecture has the desirable Lipschitz continuity, which provides a strong guarantee for the convergence analysis of our proposed solver. Furthermore, we compare this attention module architecture with other variants.

\begin{algorithm}[tb]
\small
   \caption{The processing of AttNS and AttNS-m.}
    \textbf{Input:}  A coarse step size $\Delta t_c$; The number of step $N$ which satisfied the evaluation time $T = N\Delta t_c$; A given equation $\text{d}\mathbf{u}/\text{d}t = \mathbf{f}(\mathbf{u}), \mathbf{u}(0) = \mathbf{c}_0$; A given numerical integration scheme $S$; The high-quality dataset $\mathbb{D}(\text{traj}(\mathbf{u}))$. The attention module $Q[\cdot|\phi]$; The learning rate $\eta$.
    
    \textbf{Output:} The learned attention module $Q[\cdot|\phi]$.
   
\begin{algorithmic}
\For{\text{epoch} from 0 to \text{Total epoch}}
\State\algorithmiccomment{Estimate the trajectories}
    \State Sample $\mathbf{u} = [\mathbf{u}_0,\mathbf{u}_1,...,\mathbf{u}_N] \sim \mathbb{D}(\text{traj}(\mathbf{u}))$;
    \State $\hat{\mathbf{u}}_0 \gets \mathbf{u}_0$;
    \For{$t$ from 0 to $N-1$}

        \State Calculate integration term $\hat{S} \gets S(\mathbf{f},\mathbf{u}_t,\Delta t_c)$;

        \State Calculate attention term $Q[\hat{S}|\phi]$;

        \State Get $\hat{\mathbf{u}}_{t+1}$ by Eq.(\ref{eq:mult}) {\color{blue}for AttNS-m} or Eq.(\ref{eq:ours}) {\color{red}for AttNS};

    \EndFor  

    \State\algorithmiccomment{Update the attention module}
\State $\hat{\mathbf{u}} \gets [\hat{\mathbf{u}}_1,...,\hat{\mathbf{u}}_{N}]$;
\State Calculate the loss $R_e$ by Eq.(\ref{eq:loss});
\State Update the parameters $\phi$ by $\phi \gets \phi - \eta\nabla_\phi R_e$;

\EndFor 
\State\Return the attention module $Q[\cdot|\phi]$
\end{algorithmic}
\label{alg:11}
\end{algorithm}

\begin{figure*}[t]
\centering
\includegraphics[width=0.96\linewidth]{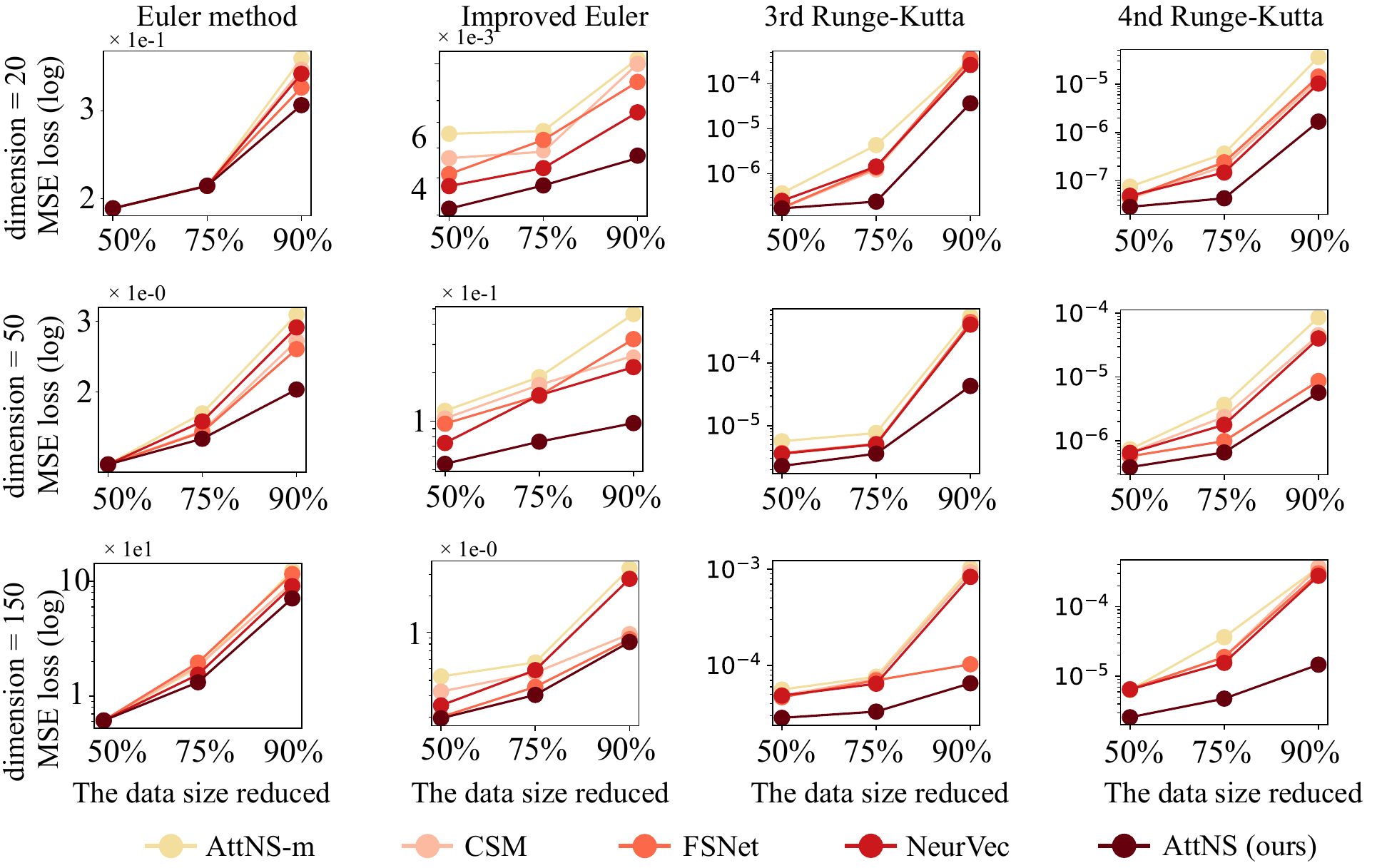}
\caption{The results of various AHS for four forward numerical solvers on the spring-mass system with different dimensions. ``50\%" denotes that we reduce the amount of training data by 50\%. \textbf{The smaller the loss, the better the performance}.}
\label{fig:1}
\end{figure*}

\textbf{(2) A baseline tailored for AHS (AttNS).} Actually, in mathematics, the attention of AttNS-m in Eq.~(\ref{eq:mult}) with multiplicative operator to calibrate the error from low-resolution term may impact the accuracy. 
The reasons are that (i) the step size $\Delta t_c$ in $\hat{S}$ used for discretization is usually not equal to the new step size $\Delta t_c^\prime := \Delta t_c\odot Q[\hat{S}|\phi]$ in AttNS-m, which isn't satisfied the mathematical correctness in Eq.~(\ref{eq:forward}); (ii) During the learning process of $Q[\hat{S}|\phi]$, this mathematical correctness further forces it to expend some optimization efforts primarily fitting a constant vector $\mathbf{I}$ with all elements being 1, which will bring negative impacts on its accuracy \cite{wang2022efficient}. Specifically, in Eq.~(\ref{eq:aihybrid}), the correction term aims to compensate for the errors of the low-resolution term, and hence the magnitude of the compensation term $\epsilon_c$ will be small. Then if multiplicative attention is used, we have $\hat{S}\Delta t_c +\epsilon_c = \hat{S}\Delta t_c\odot Q[\hat{S}|\phi]$,
i.e., $(Q[\hat{S}|\phi] - \mathbf{I})\hat{S}\Delta t_c = \epsilon_c$. Since $\|\hat{S}\Delta t_c\| \gg \|\epsilon_c\|$, we have $Q[\hat{S}|\phi] \approx \mathbf{I}$. This analysis is also supported by Fig.~\ref{fig:multi}b.

To mitigate this issue, a simple-yet-effective strategy is to normalize the attention as $\mathbf{I} + \widetilde{Q}[\hat{S}|\phi]$ and just learn the residual part $\widetilde{Q}[\hat{S}|\phi]$, and we can rewrite Eq.~(\ref{eq:mult}) as
\begin{equation} 
	\begin{aligned}
 \hat{\mathbf{u}}_{n+1} &= \hat{\mathbf{u}}_{n} + \hat{S}\Delta t_{c} {\color{blue} \odot \left( \mathbf{I} + \widetilde{Q}[\hat{S}|\phi] \right)} \\
    &= \hat{\mathbf{u}}_{n} + \hat{S}\Delta t_{c}  \underbrace{{\color{red}  + \hat{S}\Delta t_{c}\odot\widetilde{Q}[\hat{S}|\phi]}}_{\text{Additive attention}}, \\
	\end{aligned}
	\label{eq:temp111232323}
	\end{equation}	
where the term $\hat{S}\Delta t_{c}\odot\widetilde{Q}[\hat{S}|\phi]$ can be regarded as the additive attention~\cite{Wu2021FastformerAA,Cheuk2021RevisitingTO,Huang2022MRPNetSD,Li2021AdditiveAF,Gao2021MultiscaleFN} correction term. Briefly, we set the iterative formula of AttNS to be
\begin{equation}
    \text{AttNS:} \quad \hat{\mathbf{u}}_{n+1} = \hat{\mathbf{u}}_{n} + \hat{S}\Delta t_{c} {\color{red} + Q[\hat{S}|\phi]}.
    \label{eq:ours}
\end{equation}
From Eq.~(\ref{eq:temp111232323}), AttNS ensures the benefits of attention while satisfying mathematical correctness over AttNS-m, as illustrated in Fig.\ref{fig:multi}a. In terms of the formula, the primary distinction between AttNS and general AHS lies in their input, where the former is $ S(\mathbf{f}, \hat{\mathbf{u}}_{n}, \Delta t_c)$ and another is $\hat{\mathbf{u}}_{n}$. In fact, this difference is sufficient for AttNS to achieve better generalization error and ensure robustness, which will be discussed in detail in Section \ref{sec:dis}. Furthermore, the performance gap introduced by the input of attention has also been observed in conventional deep learning tasks \cite{guo2020spanet,hu2018squeeze,huang2023understanding}, further validating the rationale behind AttNS. We summarise AttNS and AttNS-m in Alg.~\ref{alg:11}.

\section{Experiments}
\label{sec:exps}

In this section, we consider two perspectives to verify the effectiveness of the proposed AttNS: (1) on different numerical solvers and (2) different differential equation benchmarks. Specifically, since AttNS is one of AHS, we use many commonly used forward numerical solvers as backbones to evaluate the enhancement achieved by AttNS, including the Euler method, Improved Euler method, 3rd and 4th order Runge-Kutta methods (see Appendix for details). On the other hand, to demonstrate that AttNS is competent for complex differential equations, we further experiment on two chaotic dynamical systems on the 4th order Runge-Kutta method, i.e., k-link pendulum and elastic pendulum. For all experiments in this section, for fair comparisons, we follow the settings in \cite{huang2022accelerating,Chen2020Symplectic} for all generations of initial conditions and metrics. In this section, we consider AttNS, AttNS-m, NeurVec~\cite{huang2023fast}, CSM~\cite{dresdner2022learning} and FSNet~\cite{xu2023} as the baselines of AHS.


\begin{figure}[t]
\centering
\includegraphics[width=0.99\linewidth]{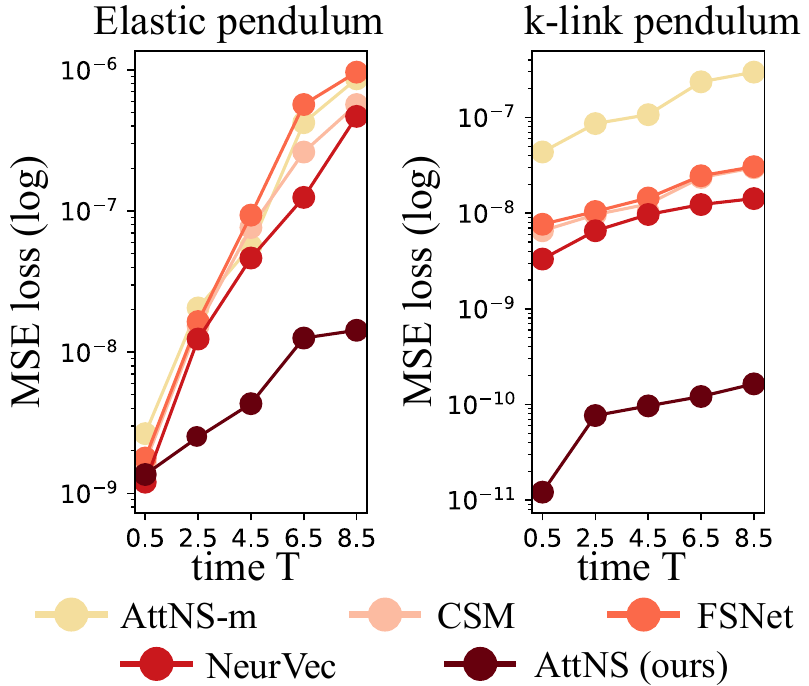}
\caption{The simulation on different chaotic systems with step size $ 1e-1$. The Mean Squared Error (MSE) loss on the test set for (a) pendulum and (b) elastic pendulum.     }
\label{fig:2}
\vspace{-0.58cm}
\end{figure}

(1) \ul{\textbf{On different numerical solvers}}. In this section, we consider a high-dimensional linear system, namely the spring-mass system, and four forward numerical solvers (see Appendix for details).
In a spring-mass system, there are $d$ masses and $d+1$ springs connecting in sequence, and they are placed horizontally with two ends connected to two fixed blocks. The corresponding ODE of this system is
\begin{equation}
\small
\frac{\text{d}}{\text{d} t}\left(\begin{array}{l}
q_i \\
p_i
\end{array}\right)=\left(\begin{array}{l}
p_i / m \\
k_i\left(q_{i-1}-q_i\right)+k_{i+1}\left(q_{i+1}-q_i\right)
\end{array}\right),
    \label{eq:spring}
\end{equation}
$i = 1,2,\cdots,d, q_0 = q_{d+1} = 0$, where $m_i$ and $k_i$ are the mass of the $t$-${\rm th}$ mass and force coefficient of the $t$-${\rm th}$ spring, respectively. The momentum and the position of $t$-${\rm th}$ mass are denoted as $p_i$ and $q_i$. We adopt the coarse step size $\Delta t = 2e-1$ for the numerical solver and the fine step size $1e-3$ for training the AHS. 

The experiment results at evaluation time $T$=20 are shown in Fig.~\ref{fig:1}. For the Euler method with low simulation accuracy, our AttNS achieves consistent performance with the state-of-the-art~(SOTA) AHS for the spring-mass system in different dimensions. For other numerical solvers with higher accuracy, AttNS can better enhance the solver than others. Even for the spring-mass system with increasing dimensions, although the difficulty of the simulation increases, AttNS can still maintain the performance. In addition, we can observe that our AttNS can achieve similar performance as the SOTA AHS with less data size, showing that we are capable of training an efficient and robust AHS with a small amount of data. AttNS-m performs poorly in most different settings, which is consistent with the analysis of Section \ref{sec:methods}.

(2) \ul{\textbf{On different chaotic systems}}. 
We reduce the amount of training data by 50\% and use two chaotic dynamical systems to verify the effectiveness of our AttNS, i.e., the elastic and k-link pendulum. The elastic pendulum considers a ball without volume connected to an elastic rod. Under the effect of gravity and force of spring~\cite{breitenberger1981elastic}, the motion of the ball will be chaotic, and its ODE is 
\begin{equation}
\small
\frac{\text{d}}{\text{d} t}\left(\begin{array}{l}
\theta \\
r\\
\dot{\theta} \\
 \dot{r} \\ 
\end{array}\right)=\left(\begin{array}{l}
\dot{\theta}\\
\dot{r}\\
\frac{1}{r}(-g \sin\theta-\dot{\theta}\dot{r})\\
r\dot{\theta}^2-\frac{k}{m}(r-l_0)+g\cos\theta
\end{array}\right),
    \label{eq:pend}
\end{equation}
where $k, m, l_0$, and $g$ are related constants. There are two variables $\theta$ and $r$ in Eq.~(\ref{eq:pend}). Specifically, $r$ is the length of the spring, and $\theta$ is the angle between the spring and the vertical axis. For k-link pendulum, it considers $K$ balls connected end to end with $K$ rods under the effect of gravity \cite{lopes2017dynamics}, and its ODE is 
\begin{equation}
    \text{d}(\mathbf{\theta},\dot{\mathbf{\theta}})/\text{d}t =  (\dot{\mathbf{\theta}},\mathbf{A^{-1}}\mathbf{b})
\end{equation}
where $\mathbf{\theta}= (\theta_1,\theta_2,\cdots,\theta_K)$ and $\theta_i$ is the angle between the $i^{\rm th}$ rod and the vertical axis. Let $\mathbf{b}=(b_1,b_2,\cdots,b_K)$ and $b_i=-\sum_{j=1}^{K}\left[c(i, j) \dot{\theta}_{j}^{2} \sin \left(\theta_{i}-\theta_{j}\right)\right]-(K-i+1) g\sin \theta_{i}$.
 $\mathbf{A}$ is a $K \times K$ matrix and the element in $\mathbf{A}_{i,j}$ is $[K - \max(i,j)+1] \cos \left(\theta_{i}-\theta_{j}\right)$.
The simulation results are shown in Fig.~\ref{fig:2}. From Fig.~\ref{fig:2}, we can observe that even for chaotic systems that are more difficult to solve, the proposed AttNS still maintains a better solution accuracy under limited data scenarios and outperforms other AHS.

\begin{figure*}
\centering
\includegraphics[width=0.99\linewidth]{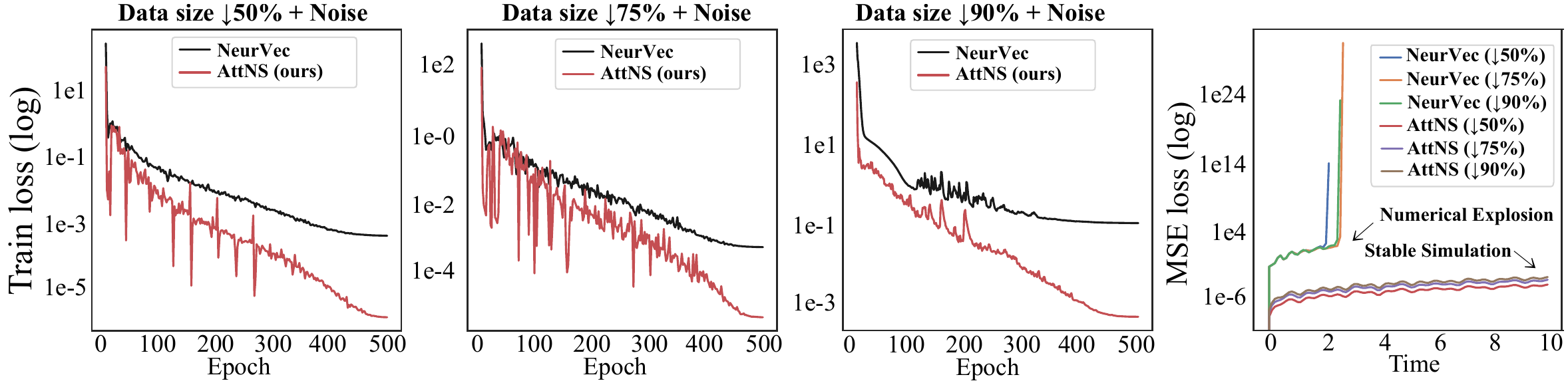}
\caption{The noise attack experiments for the elastic pendulum under different data sizes. AttNS, with the attention mechanism, can better mitigate the adverse effects of noise than other AHS. }
\label{fig:noiseattack}
\end{figure*}
\begin{table*}
  \centering
  \caption{The impact of the depth $h$ and width $d_1$ of the attention module on the simulation performance. Rel. Infe. Speed is the relative inference speed of the network based on our proposed AttNS setting, i.e., $h=2$ and $d_1 = 1024$.}
  \resizebox{0.99\hsize}{!}{
    \begin{tabular}{lllllll}
    \toprule
          & \multicolumn{2}{c}{\textbf{Spring-mass}} & \multicolumn{2}{c}{\textbf{k-link pendulum}} & \multicolumn{2}{c}{\textbf{elastic pendulum}} \\
\cmidrule{2-7}          & \multicolumn{1}{c}{\textbf{MSE loss}} & \multicolumn{1}{c}{\textbf{Rel. Infe. Speed}} & \multicolumn{1}{c}{\textbf{MSE loss}} & \multicolumn{1}{c}{\textbf{Rel. Infe. Speed}} & \multicolumn{1}{c}{\textbf{MSE loss}} & \multicolumn{1}{c}{\textbf{Rel. Infe. Speed}} \\
    \midrule
    \rowcolor{Gray} $h=2$ (ours)   &  2.58e-6     &  -      &  4.26e-9     &   -    &   5.41e-7    & - \\
    $h=3$     &  2.49e-6 ({\color{black}{$\uparrow$ 3.61\%}})    & {\color{black}{$\downarrow$ 88.49\%}}      & 4.39e-9 ({\color{black}{$\downarrow$ 2.52\%}})      & {\color{black}{$\downarrow$ 91.60\%}}      & 5.28e-7 ({\color{black}{$\uparrow$ 2.46\%}})      & {\color{black}{$\downarrow$ 90.89\%}} \\
    $h=4$     &  2.55e-6 ({\color{black}{$\uparrow$ 1.18\%}})    & {\color{black}{$\downarrow$ 93.95\%}}      & 4.42e-9 ({\color{black}{$\downarrow$ 3.62\%}})      & {\color{black}{$\downarrow$ 95.64\%}}      & 5.58e-7 ({\color{black}{$\downarrow$ 3.05\%}})      & {\color{black}{$\downarrow$ 95.33\%}} \\
    \midrule
    $d_1 = 512$   & 3.76e-5 ({\color{black}{$\downarrow$ 83.36\%}})      & {\color{black}{$\uparrow$ 27.59\%}}      & 2.56e-8 ({\color{black}{$\downarrow$ 93.14\%}})      & {\color{black}{$\uparrow$ 15.36\%}}      & 2.56e-6 ({\color{black}{$\downarrow$ 78.87\%}})      & {\color{black}{$\uparrow$ 17.25\%}} \\
    \rowcolor{Gray} $d_1=1024$ (ours) & 2.58e-6       &  -     &  4.26e-9     &   -    & 5.41e-7      & - \\
    $d_1=2048$  & 2.41e-6 ({\color{black}{$\uparrow$ 7.05\%}})       & {\color{black}{$\downarrow$ 26.48\%}}      & 3.98e-9 ({\color{black}{$\uparrow$ 7.03\%}})      & {\color{black}{$\downarrow$ 23.12\%}}      & 4.83e-7 ({\color{black}{$\uparrow$ 12.01\%}})      &  {\color{black}{$\downarrow$ 24.63\%}} \\
    \bottomrule
    \end{tabular}%
    }
  \label{tab:depth-width}%
\end{table*}%

\section{Discussion}
\label{sec:dis}

(1) \ul{\textbf{The convergence analysis of AttNS}}. Our AttNS is a numerical solver, so we need to estimate its numerical convergence. 
Now we take our AttNS with step size $k\Delta t$ and the Euler method as an example to provide the convergence analysis for the proposed AttNS in Theorem \ref{theo:con}.

\begin{theorem}\label{theo:con}
We consider ODE
$\text{d}\mathbf{u}/\text{d}t = \mathbf{f}(\mathbf{u}), \mathbf{u}(0) = \mathbf{c}_0$ and Euler method $\mathbf{u}_{n+1} = \mathbf{u}_{n} + \Delta t \mathbf{f}(\mathbf{u}_{n})$ . We assume that: (1) $\mathbf{f}$ is Lipschitz continuous with Lipschitz constant $L$, and (2) the second derivative of the true solution $\mathbf{u}$ is uniformly bounded by $M>0$, i.e., $\|\mathbf{u}''\|_\infty\leq M$ on $[0,T]$. Moreover, we assume that the attention module $Q[\cdot]$ in AttNS is Lipschitz continuous with Lipschitz constant $k\Delta t L_{\text{att}}$. For the solution of AttNS $\hat{\mathbf{u}}$ with step size $k\Delta t$, we have
\begin{equation}
    |\hat{\mathbf{u}}_N - \mathbf{u}(T)| \leq \alpha\Delta t + \beta \sqrt{\delta},
\end{equation}
where $\alpha = \frac{1}{2L}M\exp(2TL)$, $\beta = \frac{\sqrt{T}\exp(TL(1+L_{\text{att}}))}{\sqrt{L(1+L_{\text{att}})}}$ and $\delta$ is a error term about the training loss $R_e$. If the AttSlover can fit the training data well, i.e., $R_e \to 0$, the error $\delta \to 0$.
\end{theorem}
\begin{proof}
	(See Appendix \ref{appendix:theo2}).\qedhere
\end{proof}
In fact, the conditions of Theorem \ref{theo:con} are mild. First, the Lipschitz condition of the attention module $Q$ is provided by the analysis of Eq.~(\ref{eq:q}) in Section \ref{sec:dis}(2), where we prove that it can be readily satisfied. Second, the boundedness of the true solution $\mathbf{u}$ can be directly verified for multiple physical systems (see Appendix). Theorem \ref{theo:con} can also be used to analyze the convergence of AttNS. For instance, some previous works \cite{du2018gradient,jacot2018neural} reveal that under certain mild conditions, gradient descent can allow a neural network to converge to a globally optimal solution, where the loss $R_e$ in Eq.~(\ref{eq:loss}) tends to 0. Then according to Theorem \ref{theo:con}, we have $\delta \to 0$,
\begin{equation}
    |\hat{\mathbf{u}} _{N}-\mathbf{u}(T)| \leq  \underbrace{ |\mathbf{u} _{Nk}-\mathbf{u}(T)|} _{\leq \Delta t \cdot M\cdot \exp(2TL)/2L}    + \underbrace{|\hat{\mathbf{u}} _{N}-\mathbf{u} _{Nk}|} _{\leq \beta\sqrt{\delta}\to 0}.
\end{equation}
and thus 
$|\hat{\mathbf{u}}_N - \mathbf{u}(T)| = \mathcal{O}(\Delta t)$. 
In this case, the AttNS with step size $k\Delta t$ can achieve the same accuracy as the Euler method with the step size $\Delta t$, whose global truncation error is also $\mathcal{O}(\Delta t)$ \cite{butcher2016numerical}. However, the evaluation speed of AttNS is approximately $\mathcal{O}(k)$ times greater than that of the Euler method in this situation. 

(2) \ul{\textbf{The network architecture of $Q[\hat{S}|\phi]$}}.
We design the attention module structure $Q[\hat{S}|\phi]$ in Eq.~(\ref{eq:q}) primarily because, on the one hand, it has a smaller number of parameters and enables faster inference speed, which ensures the computational efficiency. For example, in Eq.~(\ref{eq:aihybrid}), if data $\mathbb{D}_f$ is generated by a fine step size $\Delta t_f$, then theoretically its evaluation speed is 
    $\mathcal{O}(\Delta t_c / [(1+\epsilon)\Delta t_f]$
 times \cite{huang2022accelerating} that of the classical method Eq.~(\ref{eq:forward}), where $\epsilon > 0$ is related to the inference speed of the neural network. Therefore, the inference speed of the neural network in AHS is important.

On the other hand, it can be proven in Theorem \ref{theo:lip} that this architecture possesses the desirable Lipschitz continuity property, which provides a strong guarantee for the convergence analysis of AttNS stated in Theorem \ref{theo:con}. We will compare this attention module structure with other variants in Section \ref{sec:abl}.
\begin{theorem}\label{theo:lip}
 For $x\in$
 definition domain $A$
, the attention module $Q(x) = \mathbf{W_2}\circ \mathbf{a} \circ \mathbf{W_1} x$
, is Lipschitz continuous. $\mathbf{a}$
 is rational activation function, i.e. $\mathbf{a}(x) = \frac{\sum_{i=1}^3a_ix^i}{\sum_{i=1}^2b_ix^i}$
, where $a_i,b_i \in \mathbb{R}$
.
\end{theorem}
\begin{proof}
	(See Appendix \ref{app:lip}).\qedhere
\end{proof}
\begin{table*}
  \centering
  \caption{(Left) The impact of skip connection and considering $\Delta t_c$ as a part of the input. (Right) The performance of other type network architectures for $Q[\hat{S}|\phi]$.}
  \vspace{-0.1cm}
  \resizebox{1.0\hsize}{!}{
    \begin{tabular}{cllc||rccc}
    \toprule
    \multicolumn{1}{l}{\textbf{Benchmark}} & \multicolumn{1}{c}{\textbf{w/ }skip} & \multicolumn{1}{c}{\textbf{w/o } $\Delta t_c$} & \multicolumn{1}{c||}{Ours} &       & \textbf{Module} & \multicolumn{1}{l}{\textbf{MSE (elastic)}} & \multicolumn{1}{l}{\textbf{MSE ($k$-link)}} \\
    \midrule
          Spring-mass &   4.57e-4 ({\color{black}{$\downarrow$ 99.44\%}})   & 3.67e-6 ({\color{black}{$\downarrow$ 29.78\%}})     &    \textbf{2.58e-6 }    &       & LSTM \cite{Hochreiter1997LongSM}  & 7.21e-06 & 1.35e-08 \\
          $K$-link &  6.36e-6 ({\color{black}{$\downarrow$ 99.93\%}})    & \textbf{4.18e-9} ({\color{black}{$\uparrow$ 1.91\%}})     &  4.26e-9       &       & Transformer \cite{liu2021swin} & 3.64e-04 & 2.65e-06 \\
          Elastic &  1.39e-4 ({\color{black}{$\downarrow$ 99.61\%}})    & 1.65e-6  ({\color{black}{$\downarrow$ 67.21\%}})    &   \textbf{5.41e-7}       &       & Ours  & \textbf{5.41e-07} & \textbf{4.26e-09} \\
    \bottomrule
    \end{tabular}%
    }
  \label{tab:wwo}
\end{table*}%

(3) \ul{\textbf{The the effectiveness of AttNS for generalization }}.
Now we extend our discussion of the data size, where some analyses have been given by Theorem \ref{theo:number} under the Euler method and some mild assumptions.
Theorem \ref{theo:number} is based on the view of Vapnik-Chervonenkis theory \cite{vapnik1999nature}.  Compared to the general AHS, AttNS requires a smaller data size to achieve the same generalization error, which is consistent with the experimental results in Fig.~\ref{fig:1}.

\begin{theorem}\label{theo:number} Let  $\text{Net}(\hat{S}|\mathbf{\phi},\mathbb{D}_f)$  and  $\text{Net}(\mathbf{u}_{n}|\mathbf{\phi},\mathbb{D}_f)$, where $\hat{S} = S(\mathbf{f}, \mathbf{u}_{n}, \Delta t_c)$, be the correction term of general AHS and AttNS, respectively. For $\epsilon>0$, when the data size is more than $N^\prime$, the empirical error of two methods satisfy $R_e(\phi|\text{AHS})\leq\epsilon$ and $R_e(\phi|\text{AttNS})\leq\epsilon$. For small enough $\epsilon_0 \ll \epsilon$ and Euler method, we have 
\begin{equation}
N(\text{AttNS})  \lesssim N(\text{AHS}),
    \label{eq:number}
\end{equation}
where $N(*)$ is the lower bound of the data size that the generalization error of method $*$ can reach $\epsilon(1-\epsilon_0)^{-1}$.
\end{theorem}

\begin{proof}
	(See Appendix \ref{appendix:theo3}).\qedhere
\end{proof}
Although our proposed method does reduce the data size required by the SOTA method, the accuracy may not be sufficient for all the problems in the natural sciences and engineering. Therefore, if we need to solve a problem that requires high accuracy, it would be better to increase the training data size as much as possible to further improve the performance of the AI-enhance solver than only relying on the algorithmic design. This is because we can observe when the complete training data is used, the improvement brought by AttNS may be less than 10\%-15\%, which is reasonable since the AttNS is designed tailored to the data-insufficient scenario and we have not explicitly maximized the performance of AttNS when the data is sufficient. In the future, we will explore improving the AttNS in the data-sufficient scenario.

(4) \ul{\textbf{The the effectiveness of AttNS for robustness }}.
 Our proposed method AttNS is inspired by the attention mechanism in ResNet. In Fig.~\ref{fig:noiseattack}, we empirically explore the robustness of AttNS by noise attack experiments under different data sizes. Injecting noise into the input is an explicit way of adding small perturbations that help us observe the robustness of the model. Specifically, in the training process, we can interfere with the training phase of AttNS by adding constant noise $\sigma = 1e-5$ to $\hat{\mathbf{u}}_n$, i.e., $\hat{\mathbf{u}}_n \gets \hat{\mathbf{u}}_n +\sigma$ for all $n$ in Eq.~(\ref{eq:ours}). Experimental results show that, compared with the baseline AHS, i.e., NeurVec, our proposed method can significantly regulate the noise, leading to smaller training loss and stable simulations. In contrast, NeurVec explodes numerically at the beginning of the simulation due to the effect of noise in all settings. 

The above experiments illustrate that the proposed AttNS is robust enough to maintain a stable solution even in limited data scenarios.

(5) \ul{\textbf{Limitation }}. From the experiment results in our paper, the proposed AttNS can achieve good enough simulation performance with less training data. However, AttNS does not completely prevent the solution from being disturbed under the inherent noise in the dataset. So we aim to alleviate the noise issue rather than solve it completely.
If we want to further mitigate this issue, we may need more elaborate network structures and training settings to enhance the effectiveness of our AttNS.

The experiments in this paper only consider the ordinary differential equation (ODE) and not the benchmark for considering partial differential equations (PDE), this is because the ODE is more conducive to some of the theoretical analyses in this paper, whereas the complex PDE is not easy to analyze. Fortunately, the AHS can also be used for PDE solv~\cite{dresdner2022learning}, so this could be used as future work to use AttNS for PDE solving.

\section{Ablation study}
\label{sec:abl}


In this section, we perform several ablation studies on AttNS. In Table \ref{tab:depth-width}, we analyze the impact of the depth $h$ and width $d_1$ for the attention module in Eq.~(\ref{eq:q}).
We observe that the depth $h$ has little effect on the MSE loss in all benchmarks, but the model's performance is positively correlated with the width $d_1$. Therefore, to increase the inference speed, we choose a sufficiently shallow depth $h=2$ and appropriate width $d_1 = 1024$ for AttNS.

Then we study popular residual structures for AttNS. From Table \ref{tab:wwo} (Left), the skip connection will bring significant negative impacts on simulation for both high-dimension linear systems and chaotic dynamical systems. This motivates us to adopt a simple stacking structure like Eq.~(\ref{eq:q}) for our AttNS. 
Moreover, we analyze the form of the input of the attention module, where we take $\hat{S}$ as input instead of the complete integration term $\hat{S}\Delta t_c$ in Eq.~(\ref{eq:q}). 

In Table \ref{tab:wwo}, we show that the input $\hat{S}$ has better performance than input $\hat{S}\Delta t_c$, except for k-link pendulum (similar performance). Furthermore, Table \ref{tab:wwo} (Right) demonstrates the performance of other network architectures for $Q[\hat{S}|\phi]$, highlighting the effectiveness of the attention module structure as defined in Eq. (\ref{eq:q}).

\section{Conclusion}
This paper discusses how to improve AHS for effective computation even in limited data scenarios. Using the dynamical system view of ResNet, we introduce the attention mechanism into the numerical solver and propose AttNS to help the generalization and robustness issue when the data is limited. Experimental results show the effectiveness of AttNS in improving various numerical solvers, where we also analyze the convergence, generalization, and robustness.

\section*{Acknowledgement}
This work was supported by National Science and Technology Major Project (No.2021ZD0111601), National Natural Science Foundation of China (No.62325605, No.62272494), and Guangdong Basic and Applied Basic Research Foundation (No.2023A1515011374), and Guangzhou Science and Technology Program (No.2024A04J6365).

\section*{Impact Statement}
The work presented in this paper aims to advance the field of scientific machine learning. Our work has many potential impacts in fields including physics, chemistry.



\nocite{langley00}

\bibliography{example_paper}
\bibliographystyle{icml2022}


\appendix
\onecolumn
\newpage
\section{The details of the proposed algorithm.}
\label{appendix:details}

\begin{table}[h]
  \centering
  \caption{\label{tab:data1}Summary of the datasets mentioned in this paper. }
  \resizebox{0.99\hsize}{!}{
    \begin{tabular}{llccccc}
    \toprule
    \multicolumn{1}{c}{Benchmark} & \multicolumn{1}{c}{Type} & Dimension & Data size & Step size  & Generative Method & Evaluation time $T$\\
    \midrule
    Spring-mass  & Train & *    & 5k   & 1e-3  & * & 20 \\
    Spring-mass  & Validation & *    & 0.1k   & 1e-5  & RK4 & 20 \\
    Spring-mass & Test & *    & 5k   & 1e-5  & RK4   & 20 \\
    \midrule
    2-link pendulum & Train & 4     &  20k     & 1e-3  & RK4   &  10\\
    2-link pendulum & Validation & 4     & 1k  & 1e-5  & RK4   & 10 \\
    2-link pendulum & Test  & 4     & 10k    & 1e-5  & RK4   & 10 \\
    \midrule
    Elastic pendulum & Train & 4     & 20k  & 1e-3  & RK4   & 50 \\
    Elastic pendulum & Validation & 4     & 1k  & 1e-5  & RK4   & 50 \\
    Elastic pendulum & Test  & 4     & 10k    & 1e-5  & RK4   & 50 \\
    \bottomrule
    \end{tabular}%
    }
\end{table}%

\noindent\textbf{The details of dataset}.
We summarize the information on training, validation, and test data for all benchmarks in this paper in Table~\ref{tab:data1}. We obtain the discrete solutions every step size up to the model time $T$ and the evaluation time for all experimental results of MSE loss in this paper is $T$. The Generative method of Spring-mass during training depends on the numerical solver used in Table 1 in main text. Moreover, the dimension of the Spring-mass system also depends on $d$ in Table 1 in main text, and if $d = 20$, the dimension is $2d = 40$. ``RK4" denotes 4th order Runge-Kutta method.


\begin{table}[h]
  \centering
  \caption{\label{tab:init}The initialization of different benchmarks. ``Uniform random'' means that the variables are sampled with uniform distribution of a given range. ``Constant'' means the variable is initialized as a constant. The dimension of $\mathbf{p}$ and $\mathbf{q}$ in Spring-mass system depends on $d$ in Table 1 in main text, and if $d = 20$, their dimension are $d = 20$.}
  \resizebox{0.99\hsize}{!}{%
    \begin{tabular}{lccccc}
    \toprule
    \multicolumn{1}{c}{\textbf{Task}} & \textbf{Variable} & \textbf{Dimension} & \textbf{Type} & \textbf{Range of initialization} & \textbf{Model input?}\\
    \midrule
    Spring-mass & $\mathbf{p}$  & *    & Uniform random & $[-2.5,2.5]^{20}$ & \checkmark\\
    Spring-mass & $\mathbf{q}$  & *    & Uniform random & $[-2.5,2.5]^{20}$& \checkmark \\
    \midrule
    Elastic pendulum & $\theta$ & 1     & Uniform random & $[0,\pi/8]$& \checkmark \\
    Elastic pendulum & $r$     & 1     & Constant & 10 & \checkmark\\
    Elastic pendulum & $\dot{\theta}$ & 1     & Constant & 0 & \checkmark\\
    Elastic pendulum & $\dot{r}$  & 1     & Constant & 0 & \checkmark\\
    Elastic pendulum & $l_0$    & 1     & Constant & 10 & \\
    Elastic pendulum & $g$     & 1     & Constant & 9.8&  \\
    Elastic pendulum & $k$    & 1     & Constant &  40& \\
    Elastic pendulum & $m$     & 1     & Constant &  1& \\
    \midrule
    2-link pendulum & $\mathbf{\theta}$ & 2     & Uniform random & $[0,\pi/8]^2$ & \checkmark\\
    2-link pendulum & $\mathbf{\dot{\theta}}$ & 2     & Constant & 0& \checkmark \\
    2-link pendulum & $m$     & 1     & Constant & 1 & \\
    2-link pendulum & $g$     & 1     & Constant & 9.8 & \\
    \bottomrule
    \end{tabular}%
    }
\end{table}%

 \noindent\textbf{Numerical solvers}. In this paper, we consider four forward numerical solvers to validate the effectiveness of our proposed AttNS, including the Euler method, Improved Euler method, 3rd and 4th order Runge-Kutta methods. In this section, we introduce these solvers. As mentioned in Eq. (2) in main text, these solvers have different integration terms $S(f, u_{n}, \Delta t)$.

\begin{table}[htbp]
  \centering
  \caption{The integration terms for different kinds of numerical solvers with step size $\Delta t$.}
    \begin{tabular}{rll}
    \toprule
    \multicolumn{1}{l}{Numerical solver} & \multicolumn{1}{l}{Integration term $S(f, u_{n}, \Delta t)$} & \multicolumn{1}{l}{Global truncation error} \\
    \midrule
    \multicolumn{1}{l}{Euler} &  $\Delta t \mathbf{f}(\mathbf{u}_{n})$ &$\mathcal{O}(\Delta t)$\\
    \multicolumn{1}{l}{Improved Euler} & $\frac{\Delta t }{2}[ \mathbf{f}(\mathbf{u}_{n}) + \mathbf{f}(\mathbf{u}_{n} + \Delta t \mathbf{f}(\mathbf{u}_{n}))]$  &$\mathcal{O}(\Delta t^2)$\\
    \multicolumn{1}{l}{3rd order Runge-Kutta}     & $\Delta t (\lambda_1K_1 + \lambda_2K_2 + \lambda_3K_3) $ &$\mathcal{O}(\Delta t^3)$\\
    \multicolumn{1}{l}{4th order Runge-Kutta}     & $\Delta t (\beta_1J_1 + \beta_2J_2 + \beta_3J_3 +\beta_4J_4) $ &$\mathcal{O}(\Delta t^4)$ \\
    \bottomrule
    \end{tabular}%
  \label{tab:dsfsdf}%
\end{table}%
For 3rd order Runge-Kutta method ,the coefficients $\lambda_1 = \lambda_3 = \frac{1}{6}$ and $\lambda_2 = \frac{2}{3}$. Besides, $K_1 = \mathbf{f}(\mathbf{u}_{n}), K_2 = \mathbf{f}(\mathbf{u}_{n}+\frac{\Delta t}{2}K_1)$ and $K_3 = \mathbf{f}(\mathbf{u}_{n} - \Delta t K_1+2\Delta tK_2)$. For the 4th order  Runge-Kutta method, $\beta_1 = \beta_4 = \frac{1}{6}$ and $\beta_2 = \beta_3 = \frac{1}{3}$. Moreover, $J_1 = \mathbf{f}(\mathbf{u}_{n}), J_2 = \mathbf{f}(\mathbf{u}_{n}+\frac{\Delta t}{2}J_1), J_3 = \mathbf{f}(\mathbf{u}_{n}+\frac{\Delta t}{2}J_2)$ and $J_4 = \mathbf{f}(\mathbf{u}_{n}+\Delta t J_3)$. Generally speaking, 4th order Runge-Kutta method has the smallest Global truncation error $\mathcal{O}(\Delta t^4)$, i.e., 4th order Runge-Kutta method has the highest accuracy. However, due to its most complex integration term $\Delta t (\beta_1J_1 + \beta_2J_2 + \beta_3J_3 +\beta_4J_4) $, its simulation speed is the slowest among the four numerical solvers in Table \ref{tab:dsfsdf}.

\section{The motivation and challenge behind the proof of our Theorems}

\subsection{Preliminaries}

This work focuses on the numerical solvers for differential equations. The stability (or robustness) and convergence are two essential metrics for numerical solvers. Specifically, we introduce the definition of the stability and convergence for ODE solvers as follows \cite{ferziger2002computational,dahlquist1956convergence}:

\textbf{Stability} refers to the sensitivity of the solver to initial conditions, parameters, and round-off errors during computation. A stable solver can produce reliable results even if the input conditions vary slightly with ``noise". If a numerical method is unstable, it may produce unreasonable results or even lead to computation failure.

\textbf{Convergence} refers to the ability of numerical computation to approach the true solution as the step size ($\Delta t$) decreases. A convergent solver will produce increasingly accurate results, and the error will tend to zero as the step size approaches zero. If a solver does not converge, it will not produce accurate results even with a small step size.

The above preliminaries on stability and convergence highlight that these two metrics are distinct for numerical solvers of differential equations. That's why we separately discuss stability and convergence by Theorem 2.1 and Theorem 5.1, respectively. Specifically,

For stability, Theorem 2.1 investigates the ability of attention mechanisms to regulate noise and enhance model robustness, inspiring us to introduce attention mechanisms in numerical solvers.

For convergence, Theorem 5.1 examines the convergence of the proposed solver by analyzing approximation errors, and provides the convergence rate of the solver under well-trained neural networks.

Numerical solvers that have both good stability and convergence can obtain good solutions in different scenarios of differential equations. Moreover, since the AI-enhanced numerical solver is still exactly a numerical solver but only enhanced by a neural network, when we only have a small size of training data, the stability and convergence of such kind of solver may not be sufficiently guaranteed, and thus we must have a new kind of AI-enhanced numerical solver that can have better theoretical guarantee on the stability and convergence such that we may better offset the adverse effect on stability and convergence than other existing solvers when the size of the training data is small. Theorem 2.1 and Theorem 5.1 respectively reveal the potential for the proposed AttNS to become a sufficiently good AI-enhanced numerical solver from these two perspectives, and we also empirically show that our AttNS performs much better than other solvers when the size of training data is small. The further theoretical analysis can be the future work to enhance the existing AttNS.

\subsection{The technical novelty and difficulty in proving theorem}

We summarize the difficulty of proving Theorem 5.1 as follows:

1. Theorem 5.1 is used to analyze the convergence of a numerical solver integrated with attention mechanism. To the best of our knowledge, we are the first to introduce attention mechanisms into numerical solvers, and hence there are not many related proof techniques in the existing literature that can be used for analyzing the convergence of the numerical solver that incorporates attention modules.

2. In proving Theorem 5.1, our goal is to make the proof approach versatile and generalizable, i.e., we want to design a unified approach that can be used to prove Theorem 5.1 for many attention modules $Q$
 that satisfy certain conditions. This requirement naturally adds difficulty on constructing our proof, but will provide us more general theoretical evidence for improving attention modules $Q$
 in the future.

3. To bridge the gap between theory and practice, it is challenging to set reasonable and mild conditions that can make Theorem 5.1 applicable for analyzing the convergence of solvers with specific empirical designs, such as our AttNS.

4. Since the solver we considered is AI-enhanced numerical, we must non-trivially connect the concepts of the numerical solver and the neural network to prove the convergence. However, the concept of the numerical solver and neural network both separately involve many different factors and perspectives in their own theoretical analyses, and the combination of these two concepts naturally introduces more factors, which greatly interferes with our target that we want to prove in Theorem 5.1. Therefore, we need to design a proof approach that can disentangle and eliminate interference factors that are irrelevant to our proof goal, enabling us to provide a clean proof that directly addresses our proof target for Theorem 5.1.

We also summarize the novelty of proving Theorem 5.1 as follows:

1. \textbf{Disentanglement.} As discussed in ``Difficulty," in order to eliminate the complex interference factors in the theoretical analysis of numerical solvers and neural networks, we skillfully disentangle the convergence of AttNS $|\hat{\mathbf{u}} _{Nk}-\mathbf{u}(T)|$
 into the convergence of the original numerical method $|\mathbf{u} _{Nk}-\mathbf{u}(T)|$
 and the performance of the attention module $|\hat{\mathbf{u}} _{N}-\mathbf{u} _{Nk}|$
 using the triangle inequality in the proof, i.e.,
\begin{equation}
    |\hat{\mathbf{u}} _{Nk}-\mathbf{u}(T)| \leq   \underbrace{|\mathbf{u} _{Nk}-\mathbf{u}(T)|} _{\leq \Delta t \cdot M \cdot \exp (2 T L) / 2 L}+ \underbrace{|\hat{\mathbf{u}} _{N}-\mathbf{u} _{Nk}|} _{\leq \beta \sqrt{\delta} \rightarrow 0},
\end{equation}
Since $|\mathbf{u} _{Nk}-\mathbf{u}(T)|$
 and module training are independent, traditional mathematical methods can be used to analyze $|\mathbf{u} _{Nk}-\mathbf{u}(T)|$
, while $|\hat{\mathbf{u}} _{N}-\mathbf{u} _{Nk}|$
 can be proven using relevant methods for neural network convergence. This disentanglement successfully bridges the convergence of AttNS and the training performance of the neural network, and may also inspire future theoretical proofs of introducing neural networks into differential equation solvers.

2. \textbf{Versatility}. As mentioned in "Difficulty," our proof should be applicable to many different attention modules. Specifically, we make our proof be versatile by considering the Lipschitz continuity of the attention module, which is a very general property that can be used to formally evaluate different and new kinds of attention module in the future. We also show in Lemma \ref{lemma:rebuttal1} that the MLP-alike attention module with a rational activation function in AttNS from Eq.(9) satisfies the Lipschitz continuity condition. The convergence of attention modules with Lipschitz conditions can be estimated uniformly using Theorem 5.1.

3. \textbf{Instructiveness.} The Lipschitz condition can guide the design of new attention modules in numerical solvers in the future and may be a necessary condition for the effectiveness of the solver. Specifically, attention modules that satisfy the Lipschitz condition, such as the "simple" MLP-alike structure with rational activation function shown in Eq.(9), can effectively solve differential equations under the proposed AttNS framework. However, some "complex" attention module structures, such as transformer and LSTM,may have poor prediction due to the fact that they generally lack the Lipschitz continuity condition, which is consistent with the experimental observation in the main text.

\newpage
\section{The proof of Theorem 5.1}
\label{appendix:theo2}
\textbf{Theorem 5.1}.
We consider ODE
$\text{d}\mathbf{u}/\text{d}t = \mathbf{f}(\mathbf{u}), \mathbf{u}(0) = \mathbf{c}_0$ and Euler method $\mathbf{u}_{n+1} = \mathbf{u}_{n} + \Delta t \mathbf{f}(\mathbf{u}_{n})$ . We assume that (1) $\mathbf{f}$ is Lipschitz continuous with Lipschitz constant $L$ and (2) the second derivative of the true solution $\mathbf{u}$ is uniformly bounded by $M>0$, i.e., $\|\mathbf{u}''\|_\infty\leq M$ on $[0,T]$. Moreover, we assume that the attention module $Q[\cdot]$ in AttNS is Lipschitz continuous with Lipschitz constant $k\Delta t L_{\text{att}}$. For the solution of AttNS $\hat{\mathbf{u}}$ with step size $k\Delta t$, we have
\begin{equation}
    |\hat{\mathbf{u}}_N - \mathbf{u}(T)| \leq \alpha\Delta t + \beta \sqrt{\delta},
\end{equation}
where $\alpha = \frac{1}{2L}M\exp(2TL)$, $\beta = \frac{\sqrt{T}\exp(TL(1+L_{\text{att}})}{\sqrt{L(1+L_{\text{att}})}}$ and $\delta$ is a error term about the training of AttSlover. If the AttSlover can fit the training data well, the error $\delta \to 0$.

\begin{lemma}
If the assumptions (1) and (2) in Theorem 5.1 hold, for Euler method $\mathbf{u}_{n+1} = \mathbf{u}_{n} + \Delta t \mathbf{f}(\mathbf{u}_{n})$, we have 
\begin{equation}
|\mathbf{u}_{Nk}-\mathbf{u}(T)|\leq \frac{M\exp(2TL)}{2L}\Delta t ,
\end{equation}
where $0=t_0<t_1<\cdots<t_{Nk}=T$ be uniform points on $[0, T]$ and $\Delta t=\frac{T}{Nk}$.
\label{lemma:1}
\end{lemma}

\begin{lemma}
For any $n\in \mathbb{N}_+$ and $x>0$, we have 
\begin{equation}
(1+x)^n \leq \exp(nx) ,
\end{equation}
\label{lemma:2}
\end{lemma}
\begin{proof}
(For Lemma \ref{lemma:2}). According to the Taylor expansion, we have 
\begin{equation}
    \exp(nx) = \sum_{i=0}^\infty\frac{(nx)^i}{i!} \geq \sum_{i=0}^\infty n^i\frac{x^i}{i!} \geq \sum_{i=0}^\infty \prod_{j=0}^{i-1} (n-j)\frac{x^i}{i!} = \sum_{i=0}^\infty C_n^i x^i = (1+x)^n
\end{equation}
\end{proof}

\begin{proof}(For Theorem 5.1.)
\begin{figure*}[h]
\centering
\includegraphics[width=1.0\linewidth]{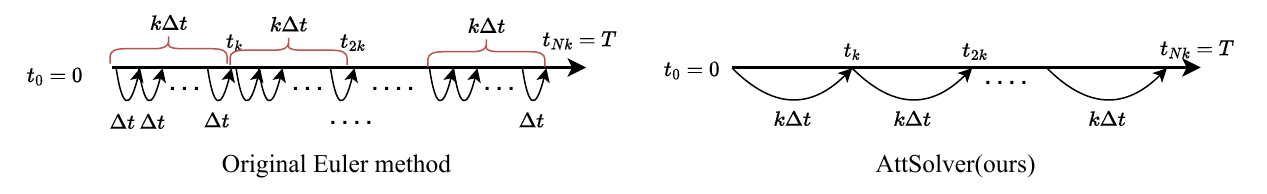}
\caption{The discretization of the original Euler method (Left) and our AttNS (Right).  }
\label{fig:discretization}
\end{figure*}
Let's consider the discretization for time interval $[0,T]$ as shown in Fig.\ref{fig:discretization}, and $E_n := \hat{\mathbf{u}}_{n} - \mathbf{u}_{nk}$
	\begin{align*}
	|\hat{\mathbf{u}}_{Nk}-\mathbf{u}(T)| &\leq |\mathbf{u}_{Nk}-\mathbf{u}(T)|+|\hat{\mathbf{u}}_{N}-\mathbf{u}_{Nk}|\tag*{Since $|a-b|\leq |a| + |b|$}\\
	&=\frac{M\exp(2TL)}{2L}\Delta t + E_N\tag*{Since Lemma \ref{lemma:1}}\\
	\end{align*}
Next we estimate the upper bound of $E_N$. For any $n\geq0$ we have 
	\begin{equation}
	\begin{aligned}
    \hat{\mathbf{u}}_{n+1}-\mathbf{u}_{k(n+1)}&=\hat{\mathbf{u}}_{n} + \mathbf{f}(\hat{\mathbf{u}}_{n}) (k\Delta t) + Q(\mathbf{f}(\hat{\mathbf{u}}_{n})|\phi)-\mathbf{u}_{k(n+1)}
    \\&=\hat{\mathbf{u}}_{n}-\mathbf{u}_{kn}+\big(\mathbf{f}(\hat{\mathbf{u}}_{n})-\mathbf{f}(\mathbf{u}_{kn})\big) (k\Delta t) + Q(\mathbf{f}(\hat{\mathbf{u}}_{n})|\phi)-Q(\mathbf{f}(\mathbf{u}_{kn})|\phi)-(k\Delta t)V_n.
	\end{aligned}
	\label{eq:temp111}
	\end{equation}	
where $V_n$ is an error term about training, i.e.,
\begin{equation}
     V_n = \frac{1}{k\Delta t}(\mathbf{u}_{k(n+1)}-\mathbf{u}_{kn} - \mathbf{f}(\mathbf{u}_{kn})k\Delta t- Q(\mathbf{f}(\mathbf{u}_{nk})|\phi)).
     \label{eq:vn}
\end{equation}

Next, from Lipschitz conditions and the triangle inequality, we have
	\begin{equation}
	\begin{aligned}
    |\hat{\mathbf{u}}_{n+1}-\mathbf{u}_{k(n+1)}|&\leq|\hat{\mathbf{u}}_{n}-\mathbf{u}_{kn}|+ L|\hat{\mathbf{u}}_{n}-\mathbf{u}_{kn}|(k\Delta t) + k\Delta tL_{\text{att}}|\mathbf{f}(\hat{\mathbf{u}}_{n})-\mathbf{f}(\mathbf{u}_{kn})|+(k\Delta t)|V_n|
    \\
    &\leq|\hat{\mathbf{u}}_{n}-\mathbf{u}_{kn}|+ L|\hat{\mathbf{u}}_{n}-\mathbf{u}_{kn}|(k\Delta t) + k\Delta tL_{\text{att}}\cdot L|\hat{\mathbf{u}}_{n}-\mathbf{u}_{kn}|+(k\Delta t)|V_n|
    \\
    &=(1+k\Delta tL+k\Delta t L_{\text{att}}L)|\hat{\mathbf{u}}_{n}-\mathbf{u}_{kn}|+(k\Delta t)|V_n|.
	\end{aligned}
	\label{eq:temp1112}
	\end{equation}	

Let $w = (1+k\Delta tL+k\Delta t L_{NV}L)$. Eq.(\ref{eq:temp1112}) can be rewritten as 
\begin{equation}
	\begin{aligned}
    |E_{n+1}|&\leq w |E_{n}|+(k\Delta t)|V_n|
    \\&\leq w(w |E_{n-1}|+(k\Delta t)|V_{n-1}|)+(k\Delta t)|V_n|\\
    &=w^2 |E_{n-1}|+w(k\Delta t)|V_{n-1}|+(k\Delta t)|V_n|
    \\&\leq w^{n+1} |E_0|+(k\Delta t)\sum_{i=0}^n w^i |V_{n-i}|=(k\Delta t)\sum_{i=0}^n w^i |V_{n-i}|,
	\end{aligned}
	\label{eq:temp11123}
	\end{equation}	

where $E_0=0$ as $E_0=\hat{\mathbf{u}}_0-\mathbf{u}_0=\mathbf{c}_0-\mathbf{c}_0=0$. Let 
\begin{equation}
     \delta =\frac{1}{N}(\|V_0\|_2^2+ \|V_1\|_2^2 +...+\|V_{N-1}\|_2^2)
 \label{eq:objective_euler}
\end{equation}

By the Cauchy inequality and Eq.(\ref{eq:temp11123}),
\begin{align*}
    |E_{N}| &= (k\Delta t)\sum_{i=0}^{N-1} w^i |V_{N-1-i}|
    \\&\leq (k\Delta t)(\sum_{i=0}^{N-1} w^{2i})^{\frac{1}{2}}(\sum_{i=0}^{N-1} |V_{N-1-i}|^2)^{\frac{1}{2}}\tag*{Since Cauchy inequality}\\
    &= (k\Delta t)\sqrt{(w^{2N}-1)/(w^2-1)}\sqrt{N\delta}\tag*{Since Eq.(\ref{eq:objective_euler})}.
\end{align*}
Next, we simplify the term $\sqrt{(w^{2N}-1)/(w^2-1)}$. Note that $w = (1+k\Delta tL+k\Delta t L_{\text{att}}L) \geq 1$, hence 
\begin{equation}
w^2 - 1 \geq w -1.
    \label{eq:fsfd}
\end{equation}
Therefore, 
\begin{align*}
    \sqrt{(w^{2N}-1)/(w^2-1)} &\leq \sqrt{[(1+k\Delta tL+k\Delta t L_{\text{att}}L)^{2N}]/[k\Delta tL+k\Delta t L_{\text{att}}L]}\tag*{Since Eq.(\ref{eq:fsfd})}
    \\&\leq \sqrt{\frac{\exp(2Nk\Delta tL(1+L_{\text{att}}))}{k\Delta tL(1+L_{\text{att}})}}\tag*{Since Lemma \ref{lemma:2}}\\
\end{align*}

Therefore, 
\begin{align*}
    |E_{N}|\leq (k\Delta t)\frac{\exp(TL(1+L_{NV}))}{\sqrt{k\Delta tL(1+L_{\text{att}})}}\sqrt{N\delta} = \frac{\sqrt{T}\exp(TL(1+L_{\text{att}}))}{\sqrt{L(1+L_{\text{att}})}}\sqrt{\delta}.
\end{align*}
Hence, we have 
\begin{equation}
    |\hat{\mathbf{u}}_N - \mathbf{u}(T)| \leq \alpha\Delta t + \beta \sqrt{\delta},
\end{equation}
where $\alpha = \frac{1}{2L}M\exp(2TL)$, $\beta = \frac{\sqrt{T}\exp(TL(1+L_{\text{att}})}{\sqrt{L(1+L_{\text{att}})}}$. Since Eq.(\ref{eq:vn}), if the AttSlover can fit the training data well, $\|V_n\|_2\to 0$, and we have $\delta \to 0$.
\end{proof}

\section{Lipschitz continuous of our $Q[\cdot]$}
\label{app:lip}

\begin{lemma}
For $x \in [a,b]$
, the polynomial function $P(x) = \sum_{k=0}^n c_k x^k$
 is Lipschitz continuous.
    \label{lemma:1}
\end{lemma}

\begin{proof}
    For all $x,y \in [a,b]$
,we have $|x^k - y^k| = (x-y)\sum_{i=0}^{k-1}x^iy^{k-1-i} \leq |x-y|\cdot k\cdot\max(a,b)^{k-1}.$
Therefore, we have
\begin{equation}
    |P(x) - P(y)| \leq \sum_{k=0}^n |c_k||x^k-y^k|\leq \underbrace{\{ \sum_{k=0}^n k\cdot |c_k|\cdot \max(a,b)^{k-1} \}}_{\text{Lipschitz constant}} \cdot |x-y|.
\end{equation}

\end{proof}

\begin{lemma}
    For $x\in$
 definition domain $A$
, the rational activation $\mathbf{a}(x) =\frac{\sum_{i=1}^3a_ix^i}{\sum_{i=1}^2b_ix^i}$
 is Lipschitz continuous.
 \label{lemma:lip}
\end{lemma}

\begin{proof}
    For all $x_1,x_2 \in A$
,Let $p(x) = \sum_{i=1}^3a_ix^i$
 and $q(x) = \sum_{i=1}^2b_ix^i$
. For all $x_1,x_2 \in A$
,we have
    \begin{align}
|\mathbf{a}(x_1) - \mathbf{a}(x_2)| &= |\frac{p(x_1)}{q(x_1)} - \frac{p(x_2)}{q(x_2)}| \\
&= |\frac{p(x_1)q(x_2)-q(x_1)p(x_2)}{q(x_1)q(x_2)}|\\
&= |\frac{p(x_1)q(x_2) - p(x_1)q(x_1) + p(x_1)q(x_1)-q(x_1)p(x_2)}{q(x_1)q(x_2)}|\\
&\leq \frac{|p(x_1)|}{|q(x_1)q(x_2)|}|q(x_2)-q(x_1)| +  \frac{|q(x_1)|}{|q(x_1)q(x_2)|}|p(x_2)-p(x_1)|.
\end{align}

Note that, since the polynomial function $p(x)$
 and $q(x)$
 are continuous in $A$
,they are bounded in $A$
. Let $\max (|p(x)|,|q(x)|) \leq M$
 and $\min (|p(x)|,|q(x)|) \geq N$
. Meoreover, from Lemma \ref{lemma:1}, for all $x_1,x_2 \in A$
, we assume $|q(x_2)-q(x_1)| \leq L_q |x_2-x_1|$
 and $|p(x_2)-p(x_1)| \leq L_p |x_2-x_1|$
.

Therefore, we have
\begin{align}
|\mathbf{a}(x_1) - \mathbf{a}(x_2)| &\leq \frac{|p(x_1)|}{|q(x_1)q(x_2)|}|q(x_2)-q(x_1)| +  \frac{|q(x_1)|}{|q(x_1)q(x_2)|}|p(x_2)-p(x_1)| \\
&\leq \frac{M}{N^2} |q(x_2)-q(x_1)| + \frac{M}{N^2} |p(x_2)-p(x_1)| \\
&\leq \underbrace{\frac{M}{N^2} \cdot (L_p+L_q)}_{\text{Lipschitz constant}} \cdot |x_1-x_2|.
\end{align}

\end{proof}

\begin{lemma}
 For $x\in$
 definition domain $A$
, the attention module shown in Eq.(9) in main text, i.e., $Q(x) = \mathbf{W_2}\circ \mathbf{a} \circ \mathbf{W_1} x$
, is Lipschitz continuous. $\mathbf{a}$
 is rational activation function, i.e. $\mathbf{a}(x) = \frac{\sum_{i=1}^3a_ix^i}{\sum_{i=1}^2b_ix^i}$
, where $a_i,b_i \in \mathbb{R}$
.
\label{lemma:rebuttal1}
\end{lemma}

\begin{proof}
    For all $x,y \in A$
, from Lemma \ref{lemma:lip}, there exist a constant $L_1$
 s.t. $||\mathbf{a}(x) - \mathbf{a}(y)|| \leq L_1 ||x-y||.$
 Therefore,
\begin{align}
||\mathbf{a}\circ \mathbf{W_1} (x) - \mathbf{a}\circ \mathbf{W_1} (y) || & \leq L_1 ||\mathbf{W_1} (x) - \mathbf{W_1} (y)||\\
&\leq L_1 ||\mathbf{W_1}||\cdot ||x-y||.
\end{align}
Moreover, we consider the attention module $Q[\cdot]$
, and we have
\begin{align}
||Q (x) - Q (y) || & = ||\mathbf{W_2}\circ \mathbf{a} \circ \mathbf{W_1} (x) -\mathbf{W_2}\circ \mathbf{a} \circ \mathbf{W_1} (y)||\\
&\leq ||\mathbf{W_1}||\cdot ||\mathbf{a}\circ \mathbf{W_1} (x) - \mathbf{a}\circ \mathbf{W_1} (y) ||\\
&\leq \underbrace{L_1 \cdot ||\mathbf{W_2}||\cdot ||\mathbf{W_1}||}_{\text{Lipschitz constant}}\cdot ||x-y||.
\end{align}
Therefore, the attention module in Eq.(9) in main text we considered is Lipschitz continuous.

\end{proof}

\subsection{About the condition $\|\mathbf{u}''\|_\infty\leq M$ on $[0,T]$.}

The condition that "the second derivative of the true solution is uniformly bounded" in Theorem 5.1 is usually easy to satisfy and our algorithm can collaborate with it. Specifically,

(1) The "second derivative" $||\mathbf{u}^{\prime\prime}||_\infty$  does not directly affect the training objective of AttNS, i.e., the AttNS can use a large step size $k\Delta t$
 (fast estimation speed) to solve the differential equations such that we can achieve a similar accuracy to the original numerical method solved by a small step size $\Delta t$
 (slow estimation speed).

The role of Theorem 5.1 is to inform us that AttNS can achieve this desired goal. Specifically, given the evaluation time $T$
 and the Lipschitz constant $L$
, if the neural network is well-trained, i.e., $\beta\sqrt{\delta}\to 0$
, then the accuracy of AttNS, $|\hat{\mathbf{u}}_{N}-\mathbf{u}(T)|$
, and the accuracy of the original numerical algorithm at small time steps $\Delta t$
, $|\mathbf{u} _{Nk}-\mathbf{u}(T)|$
, are comparable. This is because
\begin{equation}
    |\hat{\mathbf{u}} _{N}-\mathbf{u}(T)| \leq  \underbrace{ |\mathbf{u} _{Nk}-\mathbf{u}(T)|} _{\leq \Delta t \cdot M\cdot \exp(2TL)/2L}    + \underbrace{|\hat{\mathbf{u}} _{N}-\mathbf{u} _{Nk}|} _{\leq \beta\sqrt{\delta}\to 0}.
\end{equation}
On the other hand, the above inequality reveals that the upper bound $M$
 of $|\mathbf{u}^{\prime\prime}|_\infty$
 mainly affects the accuracy of both AttNS and the original numerical method. When $M$
 is large, the upper bound of their accuracy may increase, but it does not affect the goal that AttNS with large step size $k\Delta t$
 can approximate to the original numerical method at small time steps $\Delta t$
.

(2) Most "second derivative" $|\mathbf{u}^{\prime\prime}|_\infty$
 are relatively small. We extracted 1000 trajectories from the testing sets of two systems, namely the spring-chain and the 2-link pendulum, and computed their $|\mathbf{u}^{\prime\prime}|_\infty$
 by second-order difference scheme. The trajectories were generated using the RK4 method with a step size of 1e-4, which can be considered as a proxy for the true solution. The results are presented in the table below.

\begin{table}[htbp]
  \centering
  \caption{The distribution of $|\mathbf{u}^{\prime\prime}|_\infty$.}
  \resizebox{0.99\hsize}{!}{
    \begin{tabular}{l|rrrrrrrrr}
    \toprule
    \textbf{$|\mathbf{u}^{\prime\prime}|_\infty$ (spring-chain)} & \textbf{30} & \textbf{60} & \textbf{90} & \textbf{120} & \textbf{150} & \textbf{180} & \textbf{210} & \textbf{240} & \textbf{270} \\
    \textbf{The percentage of < $|\mathbf{u}^{\prime\prime}|_\infty$} & 55.35\% & 80.51\% & 91.84\% & 96.68\% & 98.57\% & 99.41\% & 99.76\% & 99.91\% & 99.97\% \\
    \midrule
    \textbf{$|\mathbf{u}^{\prime\prime}|_\infty$ (2-link pendulum)} & \textbf{5} & \textbf{10} & \textbf{15} & \textbf{20} & \textbf{25} & \textbf{30} & \textbf{35} & \textbf{40} & \textbf{45} \\
    \textbf{The percentage of < $|\mathbf{u}^{\prime\prime}|_\infty$} & 41.10\% & 65.93\% & 80.55\% & 89.48\% & 94.66\% & 97.61\% & 99.03\% & 99.71\% & 99.95\% \\
    \bottomrule
    \end{tabular}%
    }
  \label{tab:temp1}%
\end{table}%

From these tables, we observe that the distribution of $|\mathbf{u}^{\prime\prime}|_\infty$
 follows a long-tailed distribution, where the majority of $|\mathbf{u}^{\prime\prime}|_\infty$
 are relatively small. Note that the magnitude of $|\mathbf{u}^{\prime\prime}|_\infty$
 varies across different systems due to differences in dimensionality and coordinates. And only a minority of $|\mathbf{u}^{\prime\prime}|_\infty$
 are large. This observation is consistent with the findings in \cite{liang2022stiffnessaware} for other chaotic systems, i.e., three-body systems and billiard systems.

The large $|\mathbf{u}^{\prime\prime}|_\infty$
 corresponds to stiffness steps in the dynamical system, where the solution changes abruptly and may have some impact on the model training. We next investigate whether our proposed Attsovler can collaborate with these large $|\mathbf{u}^{\prime\prime}|_\infty$
.

(3) Our proposed AttNS can collaborate with large "second derivative". As the examples using the elastic pendulum system and 2-link pendulum, we selected the trajectories in the test set with the top 10\% the largest "second derivative" and recombined them into a new test set. The experimental results with RK4 (step size = 1e-1) at evaluation time 
$T$=10 are presented below.

\begin{table}[htbp]
  \centering
  \resizebox{0.5\hsize}{!}{
    \begin{tabular}{lcc}
    \toprule
    \textcolor[rgb]{ .2,  .2,  .2}{\textbf{Method}} & \multicolumn{1}{p{10.625em}}{\cellcolor[rgb]{ 1,  .992,  .98}\textcolor[rgb]{ .2,  .2,  .2}{\textbf{MSE loss (elastic)}}} & \multicolumn{1}{l}{\textcolor[rgb]{ .2,  .2,  .2}{\textbf{MSE loss (2-link)}}} \\
    \midrule
    RK4   & 4.33e-0 & 3.56e-2 \\
    NeurVec & 6.32e-5 & 3.54e-6 \\
    AttNS (ours) & \textbf{1.34e-6} & \textbf{8.54e-8} \\
    \bottomrule
    \end{tabular}%
    }
  \label{tab:ddd}%
\end{table}%
We can find that our method can still significantly outperform SOTA and traditional numerical methods in these trajectories with relatively large "second derivatives", although the performance does decrease.

\newpage
\section{The proof of Theorem 5.3.}
\label{appendix:theo3}
\textbf{Theorem 5.3}.
Let $\text{Net}(\hat{S}|\mathbf{\phi},\mathbb{D}_f)$  and  $\text{Net}(\mathbf{u}_{n}|\mathbf{\phi},\mathbb{D}_f)$, where $\hat{S} = S(\mathbf{f}, \mathbf{u}_{n}, \Delta t_c)$, be the correction term of general AHS and AttNS, respectively. For $\epsilon>0$, when the data size is more than $N^\prime$, the empirical error of two methods satisfy $R_e(\phi|\text{AHS})\leq\epsilon$ and $R_e(\phi|\text{AttNS})\leq\epsilon$. For small enough $\epsilon_0 \ll \epsilon$ and Euler method, we have 
\begin{equation}
N(\text{AttNS})  \lesssim N(\text{AHS}),
    \label{eq:number}
\end{equation}
where $N(*)$ is the lower bound of the data size that the generalization error of method $*$ can reach $\epsilon(1-\epsilon_0)^{-1}$.
\vspace{0.3cm}
\begin{lemma}
Consider the set of the models $S_k = \{f(\cdot|w),w\in \omega_k\}$. If $S_1 \subset S_2 \subset \cdots \subset S_k \subset \cdots$, 
\begin{equation}
    h_1 \leq h_2 \leq \cdots \leq h_k \leq \cdots,
\end{equation}
where $h_i$ is the Vapnik-Chervonenkis dimension of the model in set $S_i,i=1,2,\cdots$.
    \label{lemma:vclemma1}
\end{lemma}
\vspace{0.3cm}

\begin{lemma}
Let $f(x) = (\ln x + \alpha)/x$, $x\in (0,+\infty)$ and $\alpha \in \mathbb{R}$. $f(x)$ achieves its maximum value at $x = e^{1+\alpha}$, when $x\in(0,e^{1+\alpha})$, $f(x)$ rises monotonically, and when $x\in(e^{1+\alpha},+\infty)$, $f(x)$ decreases monotonically.
    \label{lemma:dandiao}
\end{lemma}
\vspace{0.3cm}

\begin{lemma}
Let $f(x) = x\exp(\frac{-a}{x})$ where $a>0$ and $x\in \mathbb{R}$. Then $f(x)$ rises monotonically. 
    \label{lemma:dandiao2}
\end{lemma}
\vspace{0.3cm}

\begin{lemma}
Let $W(x)$ be Lambert $W$ function, if $|x|<1/e$, the Taylor expansion of $W(x)$ is 
\begin{equation}
    W(x) = \sum_{n=1}^\infty\frac{(-n)^{n-1}}{n!}x^n = x - x^2 +\frac{3}{2}x^3 - \frac{8}{3}x^4+\cdots. 
\end{equation}
    \label{lemma:lam-taylor}
\end{lemma}

\begin{assumption}
For $x\in \mathbb{R}$ and a general real function $f(x)$, we consider the maps $f_1 : x\to \nabla_x f(x)^Tf(x)$ and $f_2: f(x) \to \nabla_x f(x)^Tf(x)$. For a small enough $\epsilon>0$, if $S_1$ and $S_2$ are the sets that $\forall g_1(\phi) \in S_1$ and $g_2(\phi) \in S_2$,
\begin{equation}
\|g_1(\phi) - f_1\|\leq \epsilon,\quad \|g_2(\phi) - f_2\|\leq \epsilon,
    \label{eq:1233213}
\end{equation}
we assume that $S_1 \subset S_2$. This assumption means that the complexity of the model for fitting $f_1$ is higher than that for $f_2$.
\label{assum:1}
\end{assumption}

\begin{proof}(For Theorem 6.1).

According to Vapnik-Chervonenkis theory \cite{vapnik1999nature} for regression, for any data distribution $P(x,y)$, model $\mathscr{A}(\phi)$, the generalization error $R(\phi)$ and empirical error $R_e(\phi)$, the inequality
\begin{equation}
R(\phi)\leq R_e(\phi)\left(1-c\sqrt{\delta(h_{\mathscr{A}(\phi)},N)}\right)^{-1}_+,
    \label{eq:vc}
\end{equation}
holds with probability $1-\eta$. $h_{\mathscr{A}(\phi)}$ is Vapnik-Chervonenkis dimension of model $\mathscr{A}(\phi)$, $N$ is the data size, $c$ usually is set as 1, and 
\begin{equation}
    \delta(h,N) = (h\left[\ln \frac{N}{h} +1 \right] -\ln \eta)/N.
\end{equation}
Therefore, from Eq.(\ref{eq:vc}), for the generalization error $R(\phi|\text{Neur.})$ and $R(\phi|\text{Att.})$, we have 
\begin{equation}
R(\phi|\text{Neur.})\leq R_e(\phi|\text{Neur.})\left(1-c\sqrt{\delta(h_{\text{Neur.}(\phi)},N)}\right)^{-1}_+,
    \label{eq:vcbound1}
\end{equation}
and
\begin{equation}
R(\phi|\text{Att.})\leq R_e(\phi|\text{Att.})\left(1-c\sqrt{\delta(h_{\text{Att.}(\phi)},N)}\right)^{-1}_+,
    \label{eq:vcbound2}
\end{equation}

Next, we consider how many data size can the term $\left(1-\sqrt{\delta(h_{\mathscr{A}(\phi)},N)}\right)^{-1}$ reach the accuracy $(1-\epsilon_0)^{-1}$. Let $\delta := (\ln \bar{N}+\alpha)/\bar{N}$, where $\bar{N} = N/h_{\mathscr{A}(\phi)}$ and $\alpha = 1- \ln \eta / h_{\mathscr{A}(\phi)}$. 
\begin{equation}
	\begin{aligned}
    (\ln \bar{N}+\alpha)/\bar{N} = \epsilon_0 &\iff \ln \bar{N} + \alpha = \bar{N}\epsilon_0
    \\&\iff \bar{N}\exp(-\bar{N}\epsilon_0) = \exp(-\alpha)\\
    &\iff -\epsilon_0\bar{N}\exp(-\bar{N}\epsilon_0) = -\epsilon_0\exp(-\alpha).
	\end{aligned}
	\label{eq:tempsds123}
	\end{equation}	
Therefore $-\epsilon \bar{N} = W(-\epsilon_0\exp(-\alpha))$, where $W$ is Lambert $W$ function \cite{lehtonen2016lambert}, and 
\begin{equation}
\bar{N} = -\frac{h_{\mathscr{A}(\phi)}}{\epsilon_0}\cdot W\left(-\epsilon_0\exp(\frac{\ln\eta}{h_{\mathscr{A}(\phi)}}-1)\right),
    \label{eq:lambert}
\end{equation}
From Lemma \ref{lemma:dandiao}, we know that when the data size $N\geq\bar{N}$, $\left(1-\sqrt{\delta(h_{\mathscr{A}(\phi)},N)}\right)^{-1}$ can reach the accuracy $(1-\epsilon_0)^{-1}$. From Eq.(\ref{eq:vcbound1}) and Eq.(\ref{eq:vcbound2}), when $N\geq \max\left\{N^\prime,-\frac{h_{\text{Neur.}(\phi)}}{\epsilon_0}\cdot W\left(-\epsilon_0\exp(\frac{\ln\eta}{h_{\text{Neur.}(\phi)}}-1)\right)  \right\}$ and $N\geq \max\left\{N^\prime,-\frac{h_{\text{Att.}(\phi)}}{\epsilon_0}\cdot W\left(-\epsilon_0\exp(\frac{\ln\eta}{h_{\text{Att.}(\phi)}}-1)\right)  \right\}$, we have
\begin{equation}
R(\phi|\text{Att.})\leq \epsilon(1-\epsilon_0)^{-1}, \quad R(\phi|\text{Neur.})\leq \epsilon(1-\epsilon_0)^{-1},
    \label{eq:vcbound22}
\end{equation}
holds with probability $1-\eta$. In fact, since $\epsilon_0\ll \epsilon$ and Lemma \ref{lemma:dandiao}, in Eq.(\ref{eq:lambert}), $\bar{N} \gg N^\prime$, therefore, we can set
\begin{equation}
    N(\text{AttNS}) = -\frac{h_{\text{Att.}(\phi)}}{\epsilon_0}\cdot W\left(-\epsilon_0\exp(\frac{\ln\eta}{h_{\text{Att.}(\phi)}}-1)\right),\quad N(\text{AHS}) = -\frac{h_{\text{Neur.}(\phi)}}{\epsilon_0}\cdot W\left(-\epsilon_0\exp(\frac{\ln\eta}{h_{\text{Neur.}(\phi)}}-1)\right).
    \label{eq:sdfsfs}
\end{equation}

Note that $\exp(\ln \eta /h_{\mathscr{A}(\phi)} - 1) < \exp(-1)$, i.e.,
\begin{equation}
    \left|-\epsilon_0\exp(\ln \eta /h_{\mathscr{A}(\phi)} - 1)\right| < \exp(-1).
\end{equation}
Therefore, since Lemma \ref{lemma:lam-taylor}, for Eq.(\ref{eq:lambert}), we have 
\begin{equation}
	\begin{aligned}
    -\frac{h_{\mathscr{A}(\phi)}}{\epsilon_0}\cdot W\left(-\epsilon_0\exp(\frac{\ln\eta}{h_{\mathscr{A}(\phi)}}-1)\right) &\approx -\frac{h_{\mathscr{A}(\phi)}}{\epsilon_0}\cdot \left(-\epsilon_0\exp(\frac{\ln\eta}{h_{\mathscr{A}(\phi)}}-1)\right)\\
    & = h_{\mathscr{A}(\phi)}\exp(\frac{\ln\eta}{h_{\mathscr{A}(\phi)}}-1).
	\end{aligned}
	\label{eq:tempsds1sasa23}
	\end{equation}	
Note that, for ODE
$\text{d}\mathbf{u}/\text{d}t = \mathbf{f}(\mathbf{u}), \mathbf{u}(0) = \mathbf{c}_0$, the target label of the neural network in AI-enhanced numerical solver is actually the error term for different numerical solver, which is dominated by $\mathcal{O}\left[\nabla\mathbf{f}(\mathbf{u})\mathbf{f}(\mathbf{u})\right]$. From Assumption \ref{assum:1} and Lemma \ref{lemma:vclemma1}, we have
\begin{equation}
h_{\text{Att.}(\phi)} \leq h_{\text{Neur.}(\phi)}.
\label{eq:dsfsd}
\end{equation}
Therefore, we have

\begin{align*}
    N(\text{AttNS}) &= -\frac{h_{\text{Att.}(\phi)}}{\epsilon_0}\cdot W\left(-\epsilon_0\exp(\frac{\ln\eta}{h_{\text{Att.}(\phi)}}-1)\right)\tag*{Since Eq.(\ref{eq:sdfsfs})}
    \\&\lesssim -\frac{h_{\text{Neur.}(\phi)}}{\epsilon_0}\cdot W\left(-\epsilon_0\exp(\frac{\ln\eta}{h_{\text{Neur.}(\phi)}}-1)\right)\tag*{Since Eq.(\ref{eq:dsfsd}) and Lemma \ref{lemma:dandiao2}}\\
    &=N(\text{AHS}) 
\end{align*}

    \qedhere
\end{proof}

\end{document}